\documentclass[journal]{IEEEtran}

\usepackage{color,soul} 
\definecolor{Dred}{RGB}{192,0,0} 
\definecolor{Blue}{RGB}{0,0,0} 

\usepackage{cite}

\usepackage{hyperref} 
\hypersetup{hidelinks}

\usepackage{graphicx}
\usepackage{float}
\ifCLASSINFOpdf
\else
\fi

\usepackage{amsmath,bm}
\usepackage{amssymb}

\usepackage{algorithmic}
\usepackage[ruled]{algorithm2e} 

\usepackage{array}

\newtheorem{lemma}{Lemma}
\newtheorem{proof}{Proof}

\usepackage{soul} 
\ifCLASSOPTIONcompsoc
  \usepackage[caption=false,font=normalsize,labelfont=sf,textfont=sf]{subfig}
\else
  \usepackage[caption=false,font=footnotesize]{subfig}
\fi
\usepackage{multirow}  
\usepackage{booktabs} 
\usepackage{bbding, pifont} 

\begin{document}
\title{Complementary Learning Subnetworks for Parameter-Efficient Class-Incremental Learning}

%

\author{Depeng~Li, and~Zhigang~Zeng,~\IEEEmembership{Fellow,~IEEE}
\thanks{Manuscript received Jun 15, 2023. This work was supported in part by the National Key R\&D Program of China under Grant 2021ZD0201300, in part by the Fundamental Research Funds for the Central Universities under Grant YCJJ202203012, in part by the State Scholarship Fund of China Scholarship Council under Grant 202206160045, and in part by the National Natural Science Foundation of China under Grants U1913602 and 61936004. (\textit{Corresponding author: Zhigang Zeng)}}
\thanks{Depeng Li and Zhigang Zeng are with the School of Artificial Intelligence and Automation, with the Institute of Artificial Intelligence, Huazhong University of Science and Technology, and also with the Key Laboratory of Image Processing and Intelligent Control of Education Ministry of China, Wuhan 430074, China (e-mail: dpli@hust.edu.cn; zgzeng@hust.edu.cn).}
}

%
%

\markboth{Journal of \LaTeX\ Class Files,~Vol.~XX, No.~XX, XXXX~XXXX}%
{Shell \MakeLowercase{\textit{et al.}}: A Sample Article Using IEEEtran.cls for IEEE Journals}
%



\maketitle


\begin{abstract}
In the scenario of class-incremental learning (CIL), deep neural networks have to adapt their model parameters to non-stationary data distributions, e.g., the emergence of new classes over time. However, CIL models are challenged by the well-known catastrophic forgetting phenomenon. \textcolor{Blue}{Typical methods such as rehearsal-based ones rely on storing exemplars of old classes to mitigate catastrophic forgetting, which limits real-world applications considering memory resources and privacy issues. In this paper, we propose a novel rehearsal-free CIL approach that learns continually via the synergy between two \textit{Complementary Learning Subnetworks}. Our approach involves jointly optimizing a plastic CNN feature extractor and an analytical feed-forward classifier. The inaccessibility of historical data is tackled by holistically controlling the parameters of a well-trained model, ensuring that the decision boundary learned fits new classes while retaining recognition of previously learned classes.} Specifically, the trainable CNN feature extractor provides task-dependent knowledge separately without interference; and the final classifier integrates task-specific knowledge incrementally for decision-making without forgetting. In each CIL session, it accommodates new tasks by attaching a tiny set of declarative parameters to its backbone, in which only one matrix per task or one vector per class is kept for knowledge retention. \textcolor{Blue}{Extensive experiments on a variety of task sequences show that our method achieves competitive results against state-of-the-art methods, especially in accuracy gain, memory cost, training efficiency, and task-order robustness.} Furthermore, to make the non-growing backbone (i.e., a model with limited network capacity) suffice to train on more incoming tasks, a graceful forgetting implementation on previously learned trivial tasks is empirically investigated. 
\end{abstract}

\begin{IEEEkeywords}
Class-incremental learning, catastrophic forgetting, complementary learning subnetworks, analytical solution, optimization.
\end{IEEEkeywords}

%
\IEEEpeerreviewmaketitle

\section{Introduction} \label{Sec_1}
\IEEEPARstart{T}{he} most common case in the real world is that new objects emerge in a sequence of tasks over time and it is expected to learn them immediately compared to the assumption that all classes are collected in advance and trained offline \cite{minhas2012incremental, mu2017classification, wang2022continuous}. This corresponds to incremental learning, in particular, \textit{class-incremental learning} (CIL), a model that can continuously learn new information without interfering with the previously learned knowledge \cite{de2022continual, tan2021incremental}. CIL places a single model in a dynamic environment where it must learn to adapt from a stream of tasks to make new predictions. In this scenario, however, with data of the current task accessible but none (at least the bulk) of the past, CIL is challenged by degrading performance on old classes, a phenomenon known as \textit{catastrophic forgetting} \cite{mccloskey1989catastrophic, thrun1995lifelong}. For example, despite current deep-learning models, such as convolutional neural networks (CNN), can be trained to obtain impressive performance, they would fail to retain the knowledge of previously learned classes when sequentially trained on new classes \cite{li2017learning, liu2021fast}. 

The main causes of catastrophic forgetting can be two-fold. From the network topology perspective, information acquired is maintained in model parameters (e.g., network weights). If one directly fine-tunes a well-trained model for a new task, its solution will be overwritten to satisfy the current learning objective \cite{french1999catastrophic}. As a result, the parameter-overwritten model will abruptly lose the ability to perform well on a former task. In the sight of data distribution, each task appears in sequences with non-stationary properties where unknown new classes typically emerge over time. This means that the training/test data are not independent and identically distributed (Non-IID). Consequently, it gives rise to serious decision boundary distortion, e.g., inter-class confusion problem \cite{wang2023semantic}.


Recently, numerous CIL approaches have been proposed to address forgetting by adapting model parameters continually, which can be roughly divided into three groups \cite{parisi2019continual, masana2023class, hu2022curiosity}. (1) \textit{Rehearsal-based approaches} involve fine-tuning network weights by keeping partial copies of data information from previously learned tasks in an exemplar buffer \cite{lin2022anchor, hu2022curiosity}. One straightforward implementation of this group would be to preserve pixel-level samples from each seen task for revisiting together with incoming tasks~\cite{rebuffi2017icarl, NIPS2017GEM}. However, the learning performance deteriorates when buffer size decreases, and is eventually not applicable to real-world scenarios where memory constraints or privacy and security issues~\cite{ECCV2018MAS} concern. As an alternative to storing data, generative replay yields past observations in pseudo samples at input layer~\cite{bang2021rainbow} or data representations at hidden layer \cite{liu2020generative, van2020replay}, and interleaves them when learning something new. However, current state-of-the-art generative adversarial networks (GAN) could perform well on simple datasets like MNIST, but scaling it up to more challenging problems (e.g., ImageNet) has been reported to be problematic~\cite{aljundi2019online, wang2022triple}. Actually, this transfers the stress from discriminative models to generative models instead of directly solving forgetting, which is demanding to recover past distributions cumulatively. \textcolor{Blue}{(2) \textit{Regularization-based approaches}, without accessing an exemplar buffer, incorporate additional regularization terms into the loss function to penalize changes of network weights deemed important for old tasks~\cite{PNAS2017EWC, ICML2017SI, WACV2020DMC}. That is, each network parameter is associated with the weight importance computed by different strategies. Nevertheless, the challenge is to correctly assign credit to the weights of an over-parameterized network when the number of tasks is large. (3) \textit{Architecture-based approaches} dynamically modify network components to absorb knowledge needed for novel classes, commonly relying on task-specific mask~\cite{rosenfeld2018dam, serra2018overcoming, ke2021achieving} or expansion~\cite{rusu2016progressive, CVPR2021EFT, AAAI2021PCL} operation to balance network capacity. Among them, the former generates masks with fixed parts allocated to each task and typically requires task identities to activate corresponding components at inference time. Most of them target the less challenging task-incremental learning~\cite{de2022continual} and thus do not apply to task-agnostic CIL. The latter explicitly adds a subnetwork per task by expanding branches while freezing the counterparts that solve previous tasks. A criticism of them is the network growth with the number of tasks \cite{hu2023dense}. More thorough discussions on the above methods can be found in the section of Related Work.}


\textcolor{Blue}{
In this paper, we focus on the more realistic setting in CIL, where training data of previous tasks are inaccessible and memory budgets are limited. To this end, we propose a novel CIL method from the parameter optimization perspective.} Termed CLSNet, two complementary learning subnetworks are trained in an end-to-end manner, which jointly optimizes a plastic CNN feature extractor and an analytical single-hidden layer feed-forward network (SLFN) classifier. The core idea of CLSNet is to holistically control the parameters of a well-trained model without being overwritten such that the decision boundary learned fits new classes while retaining its capacity to recognize old classes. This yields an effective learning framework amenable to CIL, as detailed later in Sec.~\ref{Sec_3}. 

Our main contributions and strengths are summarized as follows. 

\begin{itemize}
    \item[1)] We propose CLSNet, a parameter-efficient CIL method comprised of two complementary learning subnetworks. In each CIL session, it assimilates new tasks by attaching a tiny set of declarative parameters to its backbone. This corresponds to an extremely limited memory budget, e.g., only one vector per class is kept for retrieving knowledge.
    
    \item[2)] We take a drastically different approach to jointly optimize the two subnetworks (i.e., a learnable CNN feature extractor and an SLFN classifier with closed-form solutions) by combining the respective strengths of back-propagation and forward-propagation. CLSNet does not rely on buffering past data for training the subnetworks.

    \item[3)] Given a non-growing backbone, it is empirically investigated that CLSNet can innately manage the long/short-term memory of each seen task, e.g., graceful forgetting previously learned trivial tasks makes it suffice to train on incoming tasks better with limited network capacity.
    
    \item[4)] The effectiveness of our approach is demonstrated by five evaluation metrics, three types of task sequences, and sufficient representative baselines, outperforming existing methods in terms of accuracy gain, memory cost, training efficiency, as well as task-order robustness.
\end{itemize}

\section{Related work} \label{Sec_2}
\subsection{\textcolor{Blue}{Class-Incremental Learning}}

\textcolor{Blue}{
A significant body of work has studied methods for addressing catastrophic forgetting. In this section, we discuss a selection of representative CIL approaches and highlight relationships to the proposed CLSNet. Based on how task-specific information is accumulated and leveraged throughout the CIL process, prior works fall into three main categories~\cite{van2019three, hsu2018re}.}

\textcolor{Blue}{
\textbf{Rehearsal-based approaches} rely on an exemplar buffer that stores data of previous tasks for rehearsal. As a classical method, GEM~\cite{NIPS2017GEM} utilizes gradient episode memory by independently constraining the loss of each episodic memory non-increase. On this basis, LOGD~\cite{CVPR2021LOGD} specifies the shared and task-specific information in episodic memory. IL2M~\cite{ICCV2019IL2M} introduces a dual-memory strategy to store both bounded exemplar images and past class statistics. i-CTRL~\cite{tong2022incremental} is founded on compact and structured representations of old classes, characterized by a fixed architecture. As previously mentioned, leveraging a rehearsal buffer in any form (e.g., raw pixels, pseudo samples, and data representations) to retrain on all previously learned tasks is less efficient and is prohibited when considering privacy and security issues~\cite{ECCV2018MAS}. By contrast, the proposed CLSNet, a rehearsal-free method, rethinks the above limitations from a parameter/solution space perspective. Hence, it is simple yet effective to properly optimize the parameters of a model itself compared to buffering and retraining past observations cumulatively.}

\textcolor{Blue}{
\textbf{Regularization-based approaches} avoids storing data so as to prioritize privacy and alleviate memory requirements. Instead, the movement of important parameters is penalized during the training of later tasks, with parameters assumed independent to ensure feasibility. EWC \cite{PNAS2017EWC} is the pioneer of this line of work, followed by SI \cite{ICML2017SI}, and MAS \cite{ECCV2018MAS}. Additionally, DMC~\cite{WACV2020DMC} involves first training a separate model for the new classes and then using publicly available unlabeled auxiliary data to integrate the new and old models via a double distillation. OWM~\cite{NMI2019OWM} aims to find orthogonal projections of weight updates that do not disturb the connecting weights of previous tasks. However, layer-wise regularization makes it difficult to find the optimal parameters when tasks are challenging. By comparison, our approach could alleviate their demanding requirement for parameter optimization by precisely controlling the parameters of a well-trained model. Particularly, we apply the regularization into the final decision layer because the feature map extracted from deeper layers is more likely to contain task-specific information, and the deeper layer can easily forget previous knowledge \cite{tang2022learning}. Detailed differences and strengths will be discussed in Sec.~\ref{Discussion}.}

\textcolor{Blue}{
\textbf{Architecture-based approaches} dynamically modify network architectures to adsorb knowledge needed for novel tasks. DER~\cite{yan2021dynamically} allocates a new learnable feature extractor per task and augments with the previously frozen features of all sub-networks in each CIL session. Similarly, FOSTER~\cite{wang2022foster} adds an extra model compression process by knowledge distillation, which alleviates expensive model overhead. DyTox~\cite{douillard2022dytox} leverages a transformer architecture with the dynamic expansion of task tokens. Furthermore, MEMO~\cite{zhou2022model} decouples the network structure and only expands specialized building blocks. Nevertheless, architecture-based approaches generally require a substantially large number of additional parameters to assist model separation, whose expansion criterion relies on the change of the loss and thus lacks theoretical guarantees. Our method differs in building upon a non-growing backbone and can vacate the network capacity by selectively removing previously learned trivial tasks.}

\begin{figure*}[htbp]
    \centering
    \includegraphics[width=4.8in]{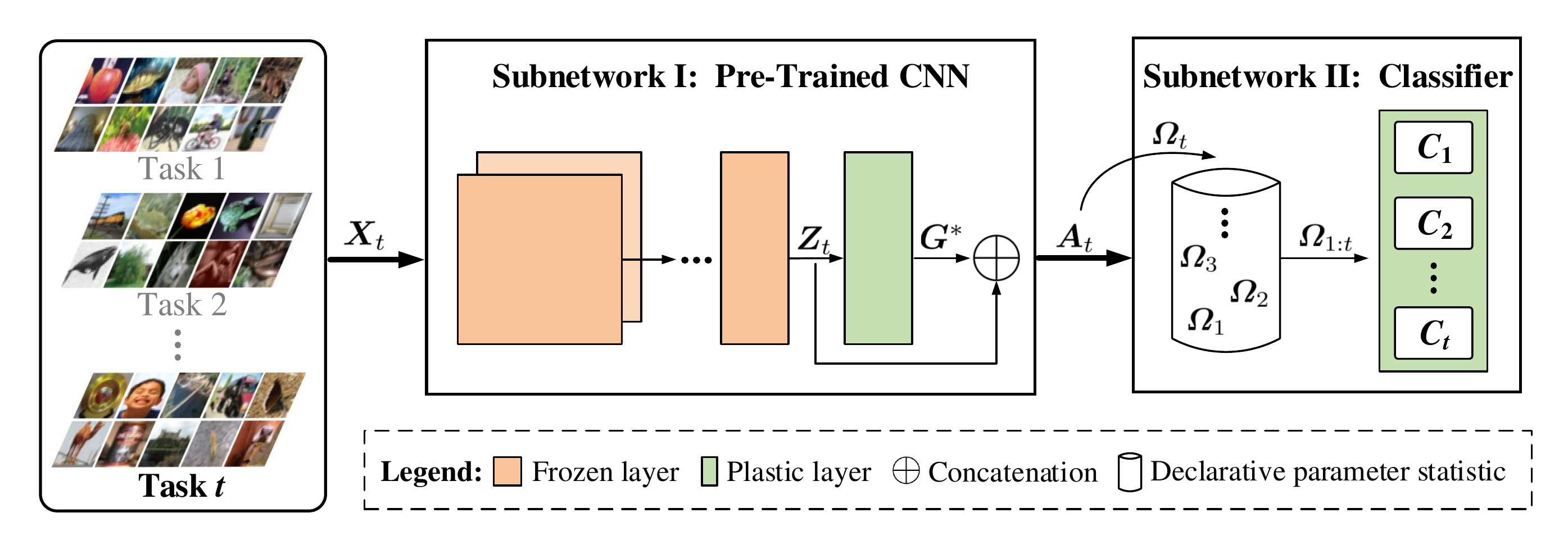}
    \caption{Overview of the proposed simple yet effective CLSNet for defying catastrophic forgetting. At training session $t$, the pre-trained CNN encodes the image inputs $\bm{X}_t$ into discriminative representations $\bm{A}_t$, which are then gradually consolidated into the analytical solution $\bm{\varOmega}_{1:t}$ in the classifier through the “replay” of declarative parameter statistic $\bm{\varOmega}=[\bm{\varOmega}_1, \bm{\varOmega}_2, \dots,\bm{\varOmega}_{t-1}]$. We jointly optimize the two subnetworks during sequential training without buffering past data. Instead, it only needs to attach a tiny set of declarative parameters with a smaller memory budget.}
    \label{Fig_CLSNet}
\end{figure*}

\subsection{Class-Incremental Learning with Pre-trained Models}
Using pre-trained models/feature extractors has been a well-received option in computer vision. Recent work on CIL also highlights the potential importance of pre-trained models~\cite{rios2022incdfm, ke2021achieving}, hypothesizing that a strong base model can provide transferable representations for novel classes. Then, the downstream CIL can be done with small adaptations~\cite{wu2022class}. PCL~\cite{AAAI2021PCL} consists of a pre-trained model as the base shared by all tasks, followed by class-specific heads. As an MLP-based CIL classifier, OWM~\cite{NMI2019OWM} takes advantage of a feature extractor that is pre-trained based on a different number of classes to analyze the raw images. REMIND~\cite{hayes2020remind} takes in an input image and passes it through frozen layers of a CNN to obtain tensor representations since early layers of CNNs have been shown to be highly transferable. Besides, some prompt-based models such as L2P~\cite{wang2022learning} and DualPrompt~\cite{wang2022dualprompt} adopt a small set of learnable parameters to instruct the representations of a pre-trained vision transformer (ViT) and store such parameters in memory space. \textcolor{Blue}{The above illustrates how pre-trained models are extensively used in the CIL community, which accommodates the real-world scenario where pre-training is usually involved as a base session.}

Although CIL starting with pre-trained models can ease the burden of catastrophic forgetting, a critical challenge is that the pre-trained knowledge needs to be adaptively leveraged for the current task while guaranteeing sufficient generalizability to future tasks \cite{zhou2023deep}. As is revealed by \cite{wang2022dualprompt, mirzadeh2022architecture}, some CIL methods still suffer from serious performance degradation given a frozen pre-trained backbone. From this perspective, it is more desirable to train a parameter-efficient CIL model holistically instead of building on top of a totally static pre-trained model. In light of this, the proposed method develops a learnable CNN as a feature extractor, followed by an analytical SLFN classifier. In each CIL session, we jointly optimize the two subnetworks to learn a new task while maintaining sufficient plasticity to incoming tasks. To our knowledge, this is yet underexplored.

\section{Problem Setting} 
This paper focuses on the most common but challenging class incremental learning (CIL) scenario \cite{van2019three, hsu2018re, masana2023class}. The problem setting can be formally defined as follows. We denote a sequence of tasks as $\bm{D}=\{\bm{D}_1, \bm{D}_2, \dots, \bm{D}_T\}$. At training session $t$, we only have access to the supervised learning datasets $\bm{D}_t=\{(\bm{X}_t,\bm{Y}_t)|\bm{X}_t\in \mathbb{R}^{N_t\times M_t}, \bm{Y}_t\in \mathbb{R}^{N_t\times C_t}\}$ of task $t$ ($t=1,2,\dots,T$), where $\bm{X}_t$ is the input, $\bm{Y}_t$ is the label, $N_t$ is the number of samples, $M_t$ and $C_t$ are the dimensions of input and output. There is no overlap between the new classes of different training sessions, i.e., $C_i\cap C_j = \varnothing (i\neq j)$. Assumed a model $\mathcal{M}(\bm{\theta}_{t-1})$~$(t\geq2)$ trained on previous task(s), parameterized by its connecting weight $\bm{\theta}_{t-1}$, the objective is to train an updated model $\mathcal{M}(\bm{\theta}_{t})$ which can incrementally recognize the newly emerging $C_t$ classes based on the datasets $\bm{D}_t$. For instance, $\mathcal{M}(\bm{\theta}_{T})$ needs to remember how to perform well on the cumulative $\sum_{t=1}^{T}{C_t}$ classes.  At test time, samples may come from any of tasks 1 to $T$, and $\mathcal{M}(\bm{\theta}_{T})$ needs to discriminate between all classes seen so far without knowing task identities. The remaining notations used throughout this paper are summarized in Table \ref{Table_Notation}.

\begin{table}[htbp]
    \caption{Notations and the descriptions}
    \label{Table_Notation}
    \centering
    \begin{tabular}{ll}
        \toprule
        Notation & Description\\ \midrule
        $\bm{\widehat{Y}}_t$  &Model output of task $t$    \\
        $g(\cdot)$            &Output function of a pre-trained CNN parameterized by $\bm{\theta}_g$   \\
        $f(\cdot)$            &Output function of the final classifier parameterized by $\bm{\theta}_f$   \\
        $g^{\prime}(\cdot)$ & The frozen layers in $g(\cdot)$ parameterized by $\bm{\theta}_{g^{\prime}}$ \\
        $g^{\prime\prime}(\cdot)$ & The plastic layer in $g(\cdot)$ parameterized by $\bm{\theta}_{g^{\prime\prime}}$ \\
        $\bm{Z}_t$ & Drifted representations of task $t$ \\
        $\bm{G}$ & Concatenation of multiple random mapping features \\
        $\bm{G}^{*}$ & Optimal representations based on $\bm{\theta}^{*}_{g^{\prime\prime}}$ \\
        $\bm{A}_t$ & Output of $g(\cdot)$ after diverse representation augmentation \\
        $\bm{\varOmega}_t$ & Declarative parameter of task $t$ \\
        $\bm{\varOmega}$ & Declarative parameter statistic over tasks seen so far \\
        $\bm{\varOmega}_{1:t}$ & Analytical solution over tasks seen so far \\
        $\bm{E}_T$ &  Prediction residual of task $T$ \\
        $\lambda_t$ & Trade-off among new and old tasks satisfying $\gamma_t=\lambda_t N_{t+1}$ \\
        $\bm{P}_t$ & Declarative parameter plasticity satisfying $\bm{\mathcal{F}}_t=\bm{P}_t\odot\bm{P}_t$ \\
        $\bm{K}_{t,c}$ & Diagonal matrix of  $\bm{\mathcal{F}}_{t,c}$ for class $c$\\
        \bottomrule
    \end{tabular}
\end{table}

\section{Methodology} \label{Sec_3}
\subsection{Overview of CLSNet} \label{Pipeline_design}
CLSNet is a parameter-efficient CIL method that learns continually via the synergy between two complementary learning subnetworks, i.e., the interplay of a plastic CNN feature extractor and an analytical SLFN classifier. Fig. \ref{Fig_CLSNet} depicts the overview of our method. The core idea is to holistically control the parameters of a well-trained model without being overwritten, such that the decision boundary learned fits new classes while retaining its capacity to recognize old classes. Learning involves two steps: 1) one subnetwork provides task-dependent knowledge separately without interference, and 2) another subnetwork incrementally integrates the knowledge specific to that task serving for decision-making without forgetting. Our method doesn't store past exemplars endlessly. Instead, it only requires the attachment of a small set of declarative parameters, making use of an extremely limited memory budget during sequential training. In addition, CLSNet can selectively forget previously learned inessential tasks to train subsequent tasks better, given a bounded network capacity.

Formally, we denote $\bm{X}_t$ as the input and $\bm{\widehat{Y}}_t$ as the corresponding model output given the current task $t$. Our CLSNet $\bm{\widehat{Y}}_t = f(g(\bm{X}_t))$ is trainable in each CIL session, where the two complementary learning subnetworks are formulated by two nested functions: $g(\cdot)$, parameterized by $\bm{\theta}_g$, consists of a CNN feature extractor; and $f(\cdot)$, parameterized by $\bm{\theta}_f$, consists of the final classifier in the decision layer. CLSNet jointly optimizes the two subnetworks by combining the respective strengths of back-propagation and forward-propagation. \textcolor{Blue}{The following explains how CLSNet addresses catastrophic forge via the synergy of two complementary learning subnetworks.}  

\subsection{Subnetwork I: Pre-Trained CNN Feature Extractor}
\label{Subnetwork_I}
The pre-trained models provide generalizable representations for the downstream tasks, weakening the difficulty of sequential training since it can be done with small adaptations. However, a critical challenge is that the datasets used to do upstream pre-training cannot cover that of CIL, i.e., no overlaps between them. Therefore, these fixed representations are intractable to be adaptively leveraged for maintaining enough generalizability to a sequence of tasks. In light of this, we incorporate the plastic components into a pre-trained model. In each CIL session, raw images are first passed through a frozen CNN $g^{\prime}(\cdot)$ and then augmented with a plastic/updatable layer $g^{\prime\prime}(\cdot)$ of the subnetwork, i.e., $g(\cdot) \leftarrow g^{\prime\prime}(g^{\prime}(\cdot))$.

\subsubsection{Initializing A Frozen CNN}
During sequential learning, we start with initializing a pre-trained CNN model as a feature extractor. This implies that the low-level representations provided by the upstream pre-training must be highly transferable across image datasets \cite{hayes2020remind}. To perform such \textit{base initialization}, we follow the common practice in prior work \cite{NMI2019OWM, AAAI2021PCL} by pre-training on a portion of the dataset offline. After that, parameters of subnetwork $g^{\prime}(\cdot; \bm{\theta}_{g^{\prime}})$ are kept fixed except for its last layer $g^{\prime\prime}(\cdot; \bm{\theta}_{g^{\prime\prime}})$. Specifically, the base initialization can be achieved by a standard ResNet-18, where we jointly initialize both the first 15 convolutional layers and 3 downsampling layers by the back-propagation algorithm on an initial subset of data, e.g., the first 200 classes of ImageNet datasets and the remaining for CIL. After it is well-trained, the former 15 convolutional layers and 2 downsampling layers are frozen, obtaining $\bm{Z}_t = g^{\prime}(\bm{X}_t; \bm{\theta}_{g^{\prime}})$; while the 18th layer is still trainable for representation augmentation, as presented below.

\subsubsection{Augmenting Transferability on A Sequence of Tasks}
The discriminative information learned for old tasks may not be sufficiently discriminative between the incoming tasks~\cite{CVPR2021EFT, AAAI2021PCL}. \textcolor{Blue}{Since the fixed representations may lack transferability on a sequence of tasks which we refer to \textit{drifted representations}, we now elaborate on how to push $\bm{Z}_t$ back to a task-optimal working state. In detail, we design a new \textit{diverse representation augmentation} strategy for unsupervised network parameter optimization, which does not require explicit label information during sequential training. Instead, the involved weights (and biases) can be reused to reproduce the optimal representations of each seen task so far, as shown in Fig. \ref{Representation_aug}.}

\begin{figure}[tbp]
    \centering
    \includegraphics[width=3.2in]{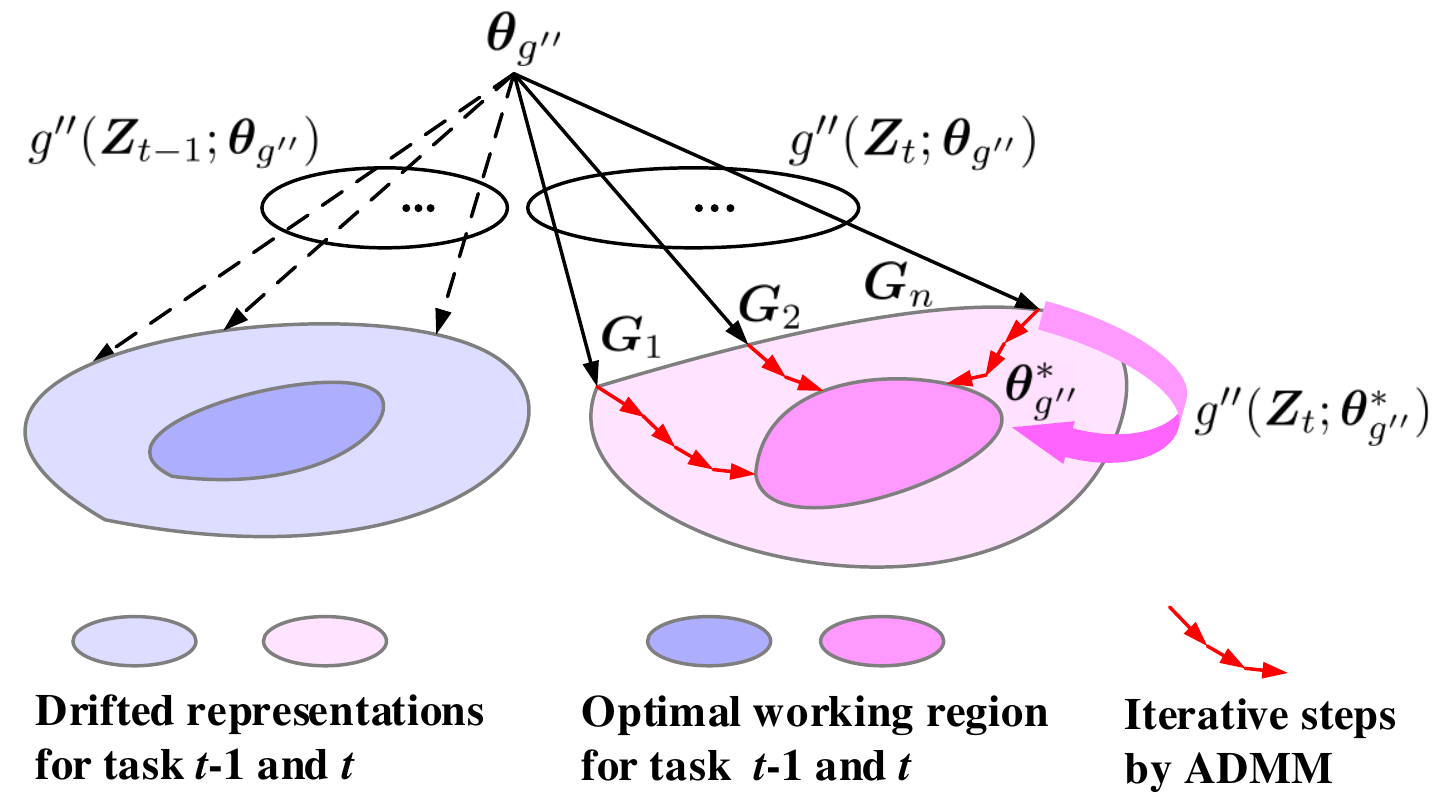}
    \caption{\textcolor{Blue}{The diverse representation augmentation strategy, aimed at improving the transferability of fixed representations provided by a pre-trained model. We first diversify knowledge by collecting multiple drifted representations; Then, we incrementally adjust the network parameters for representation augmentation such that they can recover at an optimal level.}}
    \label{Representation_aug}
\end{figure}

\textcolor{Blue}{
\textit{a) Diversifying drifted representations:} Specifically, to learn diverse and rich features potentially useful for new tasks, we first divide the last layer $g^{\prime\prime}(\cdot; \bm{\theta}_{g^{\prime\prime}})$ of the pre-trained model into $n$ groups of mapping feature nodes with each $k$ nodes:}

\begin{equation}\label{Gi}
    \bm{G}_i=\bm{Z}_t\bm{W}_i+\bm{\beta}_i, i=1,2,\dots,n
\end{equation}

\noindent where $\bm{W}_i$ and $\bm{\beta}_i$ are randomly assigned within a proper scope setting (e.g., $[-1, 1]$) \cite{TNNLS2017BLS, wang2017stochastic}, denoted by $\bm{\theta}_{g^{\prime\prime}} = \{\bm{W}_i, \bm{\beta}_i\}$. In this way, we have $\bm{G}=[\bm{G}_1, \bm{G}_2, \dots, \bm{G}_n]$ by concatenating all the \emph{n} groups of random mapping features. 

\textcolor{Blue}{
\textit{b) Restoring optimal representations:} Then, it is critical to tweak every $\bm{G}_i$ at an optimal level, achieved by slightly adjusting the network parameters $\bm{\theta}_{g^{\prime\prime}}$ to $\bm{\theta}^{*}_{g^{\prime\prime}}$. To this end, we formulate it into an optimization problem, i.e.,}

\begin{equation}\label{Theta}
    \begin{split}
    \arg\min_{\bm{\theta}^{*}_{g^{\prime\prime}}}&: p(\bm{\theta}^{*}_{g^{\prime\prime}}) + q(\bm{\theta}_{g^{\prime\prime}})  \\
        s.t.&: \bm{\theta}^{*}_{g^{\prime\prime}} - \bm{\theta}_{g^{\prime\prime}} = \bm{0} 
    \end{split}
\end{equation}

\noindent where $p(\bm{\theta}^{*}_{g^{\prime\prime}}) = \Vert \bm{G}_i\bm{\theta}^{*}_{g^{\prime\prime}} -\bm{Z}_t \Vert^2_2$,  $q(\bm{\theta}_{g^{\prime\prime}}) = \alpha\Vert\bm{\theta}_{g^{\prime\prime}}\Vert_{1}$, and $\alpha$ is a constant. The above problem denoted by lasso\cite{tibshirani1996regression} is convex, which can be solved by the alternating direction method of multipliers
(ADMM) \cite{boyd2011distributed}. The implementation of ADMM is described as the following iterative steps:

\begin{equation}\label{admm1}
    \left\{
    \begin{aligned}
        \bm{\theta}^{*}_{g^{\prime\prime}}(k+1) = & (\rho\bm{I} + \bm{G}_i^{\mathrm{T}}\bm{G}_i)^{-1}(\bm{G}_i^{\mathrm{T}}\bm{Z}_t + \\ &\rho(\bm{\theta}_{g^{\prime\prime}}(k) - \bm{u}(k)))\\
        \bm{\theta}_{g^{\prime\prime}}(k+1) = & S_{\frac{\alpha}{\rho}}(\bm{\theta}^{*}_{g^{\prime\prime}}(k+1) + \bm{u}(k)) \\
        \bm{u}(k+1) = & \bm{u}(k) + (\bm{\theta}^{*}_{g^{\prime\prime}}(k+1) - \bm{\theta}_{g^{\prime\prime}}(k+1)) 
    \end{aligned}
    \right.
\end{equation}

\noindent where $\rho > 0$ and $S$ is the soft thresholding operator, which can be defined as

\begin{equation}\label{admm2}
    S_b(a) = \left\{
    \begin{aligned}
        a-b, \quad &a >b\\
        0, \quad &\vert a \vert\leq b\\
        a+b, \quad &a < -b\\
    \end{aligned}
    \right.
\end{equation}

\textcolor{Blue}{By performing the above iterative steps, we could manage to recover the optimal representations $\bm{G}^{*}=[\bm{G}_1^{*}, \bm{G}_2^{*}, \dots, \bm{G}_n^{*}]$ in an unsupervised manner, where  $\bm{G}_i^{*}=\bm{Z}_t\bm{W}_i^{*}+\bm{\beta}_i^{*}$, $\bm{\theta}^{*}_{g^{\prime\prime}} = \{\bm{W}_i^{*}, \bm{\beta}_i^{*}\}$, and $g^{\prime\prime}(\bm{Z}_t; \bm{\theta}^{*}_{g^{\prime\prime}})$.} Using the optimized $\bm{\theta}^{*}_{g^{\prime\prime}}$ as the last layer's weights of the pre-trained model would more vividly reflect the discriminative information of each task. Finally, $g^{\prime}(\bm{X}_t; \bm{\theta}_{g^{\prime}})$ are expanded with $g^{\prime\prime}(\bm{Z}_t; \bm{\theta}^{*}_{g^{\prime\prime}})$ as a whole to augment the transferability on a sequence of tasks (see Sec.\ref{Ablation_2} for more), which is denoted as $\bm{A}_t=[\bm{Z}_t,\bm{G}^{*}]$.

\subsection{Subnetwork II: Analytical SLFN Classifier}
To guarantee the decision boundary learned fits new classes while retaining its capacity to accommodate unseen classes, we present an analytical SLFN classifier $f(\cdot; \bm{\theta}_f)$ as the read-out layer. As opposed to the commonly-used softmax layer, we derive the closed-form solution for the final single-head classifier. \textit{First}, we compute the output weights of each task, designated as \textit{declarative parameters}; \textit{Second}, we incrementally update them over all tasks seen so far for the decision-making process. In this way, $\bm{\theta}_f$ is explicitly computed into a closed-form solution and consolidated in a forward-propagation fashion, instead of error back-propagation.

\subsubsection{Computing Declarative Parameters} \label{Sec_Classifier}
Based on the optimal representations $\bm{A}_t$ yielded by the pre-trained CNN feature extractor, the output function of the classifier is $f(\bm{A}_t; \bm{\theta}_f)$, which can be written as the matrix form $\bm{Y}_t=\bm{A}_t\bm{\varOmega}_t$. In this way, the declarative parameter $\bm{\varOmega}_t$ of task $t$ can be easily obtained by the Moore-Penrose generalized inverse, i.e., $\bm{\varOmega}_t = \bm{A}_t^{+}\bm{Y}_t$. The computational complexity introduced by the matrix inversion operation can be further circumvented by

\begin{equation}\label{A+}
    \bm{A}_t^+=(\rho\bm{I}+\bm{A}_t^\mathrm{T}\bm{A}_t)^{-1}\bm{A}_t^\mathrm{T}
\end{equation}

\noindent where $\bm{A}_t^\mathrm{T}$ is the transpose of $\bm{A}_t$ and the same constraint term $\rho$ in Eq. (\ref{admm1}) is employed to find the pseudo-inverse when the original generalized inverse is under the ill condition. Without loss of generality, the declarative parameter of task $t$ can be denoted as

\begin{equation}\label{varOmega}
    \bm{\varOmega}_t=(\rho\bm{I}+\bm{A}_t^\mathrm{T}\bm{A}_t)^{-1}\bm{A}_t^\mathrm{T}\bm{Y}_t
\end{equation}

Meanwhile, we preserve each $\bm{\varOmega}_t$ for knowledge retention, termed declarative parameter statistic $\bm{\varOmega} = [\bm{\varOmega}_1, \bm{\varOmega}_2, \dots, \bm{\varOmega}_t]$, meaning that our method keeps only one matrix per task or one vector per class, which is memory-efficient compared to the exemplar budget used in rehearsal-based methods.

\subsubsection{Decision-Making Without Forgetting} 
Aided by the declarative parameter statistic, we update the classifier to retain the knowledge of previously learned tasks. Suppose there have been $T-1$ $(T\geq 2)$ tasks trained so far and we have the resulting declarative parameter statistic
$\bm{\varOmega}=[\bm{\varOmega}_1, \dots,\bm{\varOmega}_{T-1}]$. Here we refer to $\bm{\varOmega}_{1:T-1}$ as the final classifier's \textit{closed-form solution} that has consolidated the previously declarative parameter $\bm{\varOmega}_t$ $(t=1,2,\dots, T-1$), and thus the classifier's output $\bm{\widehat{Y}}_t=\bm{A}_t\bm{\varOmega}_{1:T-1}$. When the $T$th task arrives, the goal is to enrich the $\bm{\varOmega}_{1:T-1}$ as $\bm{\varOmega}_{1:T}$ such that it can accommodate the new task without forgetting previous ones. We formulate this into the following objective function:

\begin{equation}\label{Obj_fun}
    \begin{split}
        \bm{J} = &\frac{1}{2N_T}\Vert\bm{E}_T\Vert_2^2+\frac{1}{2}\sum_{t=1}^{T-1}\Vert \sqrt{\lambda_t}\bm{P}_{t}\odot(\bm{\varOmega}_{1:T} - \bm{\varOmega}_{t})\Vert_2^2 \\
        & + \frac{1}{2N_T}\Vert \bm{\varOmega}_{1:T} - \bm{\varOmega}_{1:T-1} \Vert_2^2
    \end{split}
\end{equation}

\noindent where $\bm{E}_T=\bm{A}_T\bm{\varOmega}_{1:T}-\bm{Y}_T$ is the prediction residual, $\odot$ is the element-wise product, $\lambda_t$ is the trade-off that controls how important old tasks are compared to task $T$, and $\bm{P}_t$ matters the \textit{declarative parameter plasticity} that is discussed in the next part (see Sec. \ref{Plasticity}). \textcolor{Blue}{Overall, the first term guarantees the sound performance of current task $\{\bm{X}_T, \bm{Y}_T\}$; the second term preserves the previously acquired knowledge; and the third term counterbalances systematic bias favoring tasks learned earlier on, which is overlooked in some regularization-based methods \cite{PNAS2017EWC, ICML2017SI}. We discuss the differences in detail later in Sec.~\ref{Discussion} after we have introduced our method properly.}

Now, we analytically solve Eq. (\ref{Obj_fun}) to obtain the closed-form solution of the final classifier, which is parameter-efficient compared to conventional gradient descent methods. To complete the derivation process, we take the derivative of the objective function with respect to $\bm{\varOmega}_{1:T}$:

\begin{equation}\label{partial}
    \begin{split}
        \frac{\partial\bm{J}}{\partial\bm{\varOmega}_{1:T}}=&\frac{1}{N_T}\bm{A}_T^\mathrm{T}(\bm{A}_T\bm{\varOmega}_{1:T} - \bm{Y}_T)+\sum_{t=1}^{T - 1}\lambda_t\bm{\mathcal{F}}_t\odot(\bm{\varOmega}_{1:T} - \bm{\varOmega}_t) \\
        & + \frac{1}{N_T}(\bm{\varOmega}_{1:T} - \bm{\varOmega}_{1:T-1})
    \end{split}
\end{equation}

\noindent where $\bm{\mathcal{F}}_t=\bm{P}_t\odot\bm{P}_t$ is the element-wise product of declarative parameter plasticity $\bm{P}_t$. By setting ${\partial\bm{J} \mathord{\left/{\vphantom {1 2}} \right.\kern-\nulldelimiterspace} \partial\bm{\varOmega}_{1:T}} = 0$, we have

\begin{equation}\label{substitution}
    \begin{split}
        & \bm{A}_T^\mathrm{T} \bm{A}_T\bm{\varOmega}_{1:T} + \sum_{t=1}^{T - 1}\gamma_t\bm{\mathcal{F}}_t\odot\bm{\varOmega}_{1:T} + \bm{\varOmega}_{1:T} \\ = & \bm{A}_T^\mathrm{T}\bm{Y}_T + \sum_{t=1}^{T - 1}\gamma_t\bm{\mathcal{F}}_t\odot\bm{\varOmega}_t + \bm{\varOmega}_{1:T-1}
    \end{split}
\end{equation}

\noindent  where $\gamma_t=\lambda_t N_{t+1}$. However, it is intractable to directly obtain $\bm{\varOmega}_{1:T}$ due to the coexistence of matrix product and element-wise product in Eq. (\ref{substitution}). Here, we employ a partitioned matrix diagonalization (PMD) strategy for problem-solving. Specifically, $\bm{\varOmega}_{1:T}$, $\bm{\mathcal{F}}_t$, and $\bm{Y}_T$ are first partitioned by columns, i.e., the corresponding entries for each class. Then, we have $\bm{\varOmega}_{1:T}=[\bm{\varOmega}_{1:T,1},\dots,\bm{\varOmega}_{1:T,c},\dots,\bm{\varOmega}_{1:T,C_T}]$, $\bm{\mathcal{F}}_t=[\bm{\mathcal{F}}_{t,1},\dots,\bm{\mathcal{F}}_{t,c},\dots,\bm{\mathcal{F}}_{t,C_t}]$, and $\bm{Y}_T=[\bm{Y}_{T,1},\dots,\bm{Y}_{T,c},\dots,\bm{Y}_{T,C_T}]$. Finally, we diagonalize $\bm{\mathcal{F}}_{t,c}$ into $\bm{K}_{t,c}$ for $c=1,2,\dots,C_t$ and Eq. (\ref{substitution}) can be rewritten as

\begin{equation}\label{combination}
    \begin{split}
        &\bigg(\bm{A}_T^\mathrm{T}\bm{A}_T + \sum_{t=1}^{T - 1}\gamma_t\bm{K}_{t,c} + \bm{I}\bigg)\bm{\varOmega}_{1:T,c} \\
        = &\bm{A}_T^\mathrm{T}\bm{Y}_{T,c} + \sum_{t=1}^{T - 1}\gamma_t\bm{K}_{t,c}\bm{\varOmega}_{t,c} + \bm{\varOmega}_{1:T-1,c}
    \end{split}
\end{equation}

\noindent Note that there are only matrix product operations in the above analytical expression, compared to Eq. (\ref{substitution}). Denote by $\bm{\varOmega}_{1:T,c}$ in $\bm{\varOmega}_{1:T}=[\bm{\varOmega}_{1:T,1},\dots,\bm{\varOmega}_{1:T,c},\dots,\bm{\varOmega}_{1:T,C_T}]$ the closed-form solution for $c$th class ($c=1,2,\dots,C_T$), which can be recursively updated as

\begin{equation}\label{varOmega_cc}
    \begin{split}
    \bm{\varOmega}_{1:T,c} = & \bigg(\bm{A}_T^\mathrm{T}\bm{A}_T+\sum_{t=1}^{T - 1}\gamma_t\bm{K}_{t,c} + \bm{I} \bigg)^{-1} \\
    &\times \bigg(\bm{A}_T^\mathrm{T}\bm{Y}_{T,c}+\sum_{t=1}^{T - 1}\gamma_t\bm{K}_{t,c}\bm{\varOmega}_{t,c} + \bm{\varOmega}_{1:T-1,c} \bigg)
    \end{split}
\end{equation}

\noindent where $\bm{K}_{t,c}$ and $\bm{Y}_{T,c}$ are the counterparts of $\bm{P}_t$ and $\bm{Y}_{T}$. Hence, the final classifier with closed-form solution $\bm{\widehat{Y}}_t=\bm{A}_t\bm{\varOmega}_{1:T}$ can recognize any of tasks 1 to $T$.
 
\subsubsection{Determining Declarative Parameter Plasticity} \label{Plasticity} 
We now discuss how to determine the declarative parameter plasticity, i.e., $\bm{\mathcal{F}}_t=\bm{P}_t\odot\bm{P}_t$, which is indispensable to guide the update of the final classifier. Inspired by the online Bayesian learning \cite{LP, EP}, we probe into some properties from a probabilistic perspective. In this way, the entries of closed-form solution $\bm{\varOmega}_{1:T}$ mean finding their most probable values given some data $\bm{D}=\{\bm{D}_1,\bm{D}_2,\dots,\bm{D}_T\}$. Hence, we can formulate the conditional probability $p(\bm{\varOmega}_{1:T}|\bm{D})$ as

\begin{equation}\label{online_Bayes}
    \begin{split}
        &p(\bm{\varOmega}_{1:T}|\bm{D})\\
        =&\frac{p(\bm{\varOmega}_{1:T},\bm{D}_T|\bm{D}_1,\dots,\bm{D}_{T-1})p(\bm{D}_1,\dots,\bm{D}_{T-1})}{p(\bm{D}_T|\bm{D}_1,\dots,\bm{D}_{T-1})p(\bm{D}_1,\dots,\bm{D}_{T-1})}\\
        =&\frac{p(\bm{D}_T|\bm{\varOmega}_{1:T},\bm{D}_1,\dots,\bm{D}_{T-1})p(\bm{\varOmega}_{1:T}|\bm{D}_1,\dots,\bm{D}_{T-1})}{p(\bm{D}_T|\bm{D}_1,\dots,\bm{D}_{T-1}))}\\
        =&\frac{p(\bm{D}_T|\bm{\varOmega}_{1:T})p(\bm{\varOmega}_{1:T}|\bm{D}_1,\dots,\bm{D}_{T-1})}{p(\bm{D}_T)}
    \end{split}
\end{equation}

\noindent Then, Eq. (\ref{online_Bayes}) is further processed with logarithm

\begin{equation}\label{log_p}
    \begin{split}
        &\log p(\bm{\varOmega}_{1:T}|\bm{D})\\
        = &\log p(\bm{D}_T|\bm{\varOmega}_{1:T})+\log p(\bm{\varOmega}_{1:T}|\bm{D}_1,\dots,\bm{D}_{T-1})\\
        & -\log p(\bm{D}_T)
    \end{split}
\end{equation}

\noindent where the log-likelihood probability $\log p(\bm{D}_T|\bm{\varOmega}_{1:T})$ is the task-specific objective function for task $T$, $\log p(\bm{D}_T)$ is a parameter-free constant, and the posterior distribution $\log p(\bm{\varOmega}_{1:T}|\bm{D}_1,\dots,\bm{D}_{T-1})$ containing information of previous tasks is the core to knowledge retention. Hence,

\begin{equation}\label{argmaxmin}
    \begin{split}
        &\bm{\varOmega}_{1:T}\\
        =&\arg \max_{\bm{\varOmega}_{1:T}}\log p(\bm{\varOmega}_{1:T}|\bm{D})\\
        =&\arg \max_{\bm{\varOmega}_{1:T}}\log p(\bm{D}_T|\bm{\varOmega}_{1:T})
        +\log p(\bm{\varOmega}_{1:T}|\bm{D}_1,\dots,\bm{D}_{T-1})\\
        \triangleq&\arg \min_{\bm{\varOmega}_{1:T}}\frac{1}{2N_T}\Vert\bm{E}_T\Vert_2^2-\log p(\bm{\varOmega}_{1:T}|\bm{D}_1,\dots,\bm{D}_{T-1})
    \end{split}
\end{equation}

\noindent Note that the above posterior distribution in Eq. (\ref{argmaxmin}) cannot be calculated directly but can be tackled by the Laplace approximation~\cite{LP, mackay1992practical}. On this basis, we would decompose the distribution into multiple terms, one for each task, i.e., $\log p(\bm{\varOmega}_{1:T}|\bm{D}_1,\dots,\bm{D}_{T-1})\approx \sum_{t=1}^{T-1}\log p(\bm{\varOmega}_{1:T}|\bm{D}_t)$ where each $p(\bm{\varOmega}_{1:T}|\bm{D}_t)$ can be sequentially approximated as a Gaussian distribution with its mean given by the $\bm{\varOmega}_{t}$ and variance given by $-\bm{H}_t^{-1}$. The Hessian matrix $\bm{H}_t=\nabla^2_{\bm{\varOmega}_{1:T}}h(\bm{\varOmega}_{1:T})$ corresponds to $\bm{\varOmega}_{1:T} = \bm{\varOmega}_t$ in the Taylor series. 



However, the computation of $\bm{H}_t$ may be complicated and time-consuming. The properties of Fisher information matrix \cite{FIM1, FIM2} could be an alternative. Specifically, the Fisher information matrix $\bm{F}$ of conditional probability $p_{\bm{\pi}}(\bm{z})$ with respect to a vector $\bm{\pi}$ is defined below

\begin{equation}\label{FIM}
    \bm{F}=\mathbb{E}_{\bm{z}}\bigg[\nabla_{\bm{\pi}}\log p_{\bm{\pi}}(\bm{z})\nabla_{\bm{\pi}}\log p_{\bm{\pi}}(\bm{z})^\mathrm{T}\bigg]
\end{equation}

\noindent and the Hessian matrix $\bm{H}$ of $\log p_{\bm{\pi}}(\bm{z})$ is given by

\begin{equation}\label{Hessian}
    \bm{H}=\nabla_{\bm{\pi}}^2\log p_{\bm{\pi}}(\bm{z})
\end{equation}

\begin{lemma}\cite{FIM1}\label{H&F}
    The negative expected of the Hessian matrix $\bm{H}$ of log-likelihood is equal to the Fisher information matrix $\bm{F}$.
\end{lemma}

\begin{proof}\label{pfoof}
    Taking expectation with respect to Eq. (\ref{Hessian}), we have
    \begin{equation}
        \begin{split}
            &\mathbb{E}_{\bm{z}}\bigg[\nabla_{\bm{\pi}}^2\log p_{\bm{\pi}}(\bm{z})\bigg]\\
            =&\mathbb{E}_{\bm{z}}\bigg[\frac{\nabla_{\bm{\pi}}^2p_{\bm{\pi}}(\bm{z})}{p_{\bm{\pi}}(\bm{z})}-\bigg(\frac{\nabla_{\bm{\pi}}p_{\bm{\pi}}(\bm{z})}{p_{\bm{\pi}}(\bm{z})}\bigg)\bigg(\frac{\nabla_{\bm{\pi}}p_{\bm{\pi}}(\bm{z})}{p_{\bm{\pi}}(\bm{z})}\bigg)^\mathrm{T}\bigg]\\
            =&\mathbb{E}_{\bm{z}}\bigg[\frac{\nabla_{\bm{\pi}}^2p_{\bm{\pi}}(\bm{z})}{p_{\bm{\pi}}(\bm{z})}\bigg]-\mathbb{E}_{\bm{z}}\bigg[\bigg(\frac{\nabla_{\bm{\pi}}p_{\bm{\pi}}(\bm{z})}{p_{\bm{\pi}}(\bm{z})}\bigg)\bigg(\frac{\nabla_{\bm{\pi}}p_{\bm{\pi}}(\bm{z})}{p_{\bm{\pi}}(\bm{z})}\bigg)^\mathrm{T}\bigg]\\
            =&\int\frac{\nabla_{\bm{\pi}}^2p_{\bm{\pi}}(\bm{z})}{p_{\bm{\pi}}(\bm{z})}p_{\bm{\pi}}(\bm{z})d\bm{z}-\mathbb{E}_{\bm{z}}\bigg[\nabla_{\bm{\pi}}\log p_{\bm{\pi}}(\bm{z})\nabla_{\bm{\pi}}\log p_{\bm{\pi}}(\bm{z})^\mathrm{T}\bigg]\\
            =&\nabla_{\bm{\pi}}^2\int p_{\bm{\pi}}(\bm{z})d\bm{z}-\bm{F}\\
            =&-\bm{F}
        \end{split}
    \end{equation}
\end{proof}

\noindent This completes the proof of \textit{Lemma 1}. $\hfill\square$

Hence, we have $\bm{F}=-\mathbb{E}_{\bm{z}}\big[{\bm{H}}\big]$, which is a surrogate to the Hessian matrix. Besides, it can further be computed from the squares of first-order gradient alone in the empirical form \cite{FIM1, FIM2}. That is, $\bm{\mathcal{F}}_t=-\mathbb{E}_{\bm{D}_t}\big[{\bm{H}_t}\big]$ when given the training data $\bm{D}_t=\{(\bm{X}_t,\bm{Y}_t)\}$, in which $\bm{X}_t=\{\bm{x}_1,\dots,\bm{x}_p,\dots,\bm{x}_{N_t}\}$ and $\bm{Y}_t=\{\bm{y}_1,\dots,\bm{y}_p,\dots,\bm{y}_{N_t}\}$. Therefore, the declarative parameter plasticity can be formulated by

\begin{equation}\label{myEFIM}
    \bm{\mathcal{F}}_t\!=\!\frac{1}{N_t}\sum_{p=1}^{N_t}\nabla_{\bm{\varOmega}_{1:T}}\log p(\bm{\varOmega}_{1:T}|\bm{x}_p)\nabla_{\bm{\varOmega}_{1:T}}\log p(\bm{\varOmega}_{1:T}|\bm{x}_p)^\mathrm{T}
\end{equation}

\noindent where $\bm{\mathcal{F}}_t=\bm{P}_t\odot\bm{P}_t$ encourages the new task to update the closed-form solution according to how plastic are these declarative parameters of prior tasks. We summarize the training procedure of the proposed CLSNet in \textbf{Algorithm \ref{alg_1}}.

\begin{algorithm}[tbp]
    \label{alg_1}
    \caption{Training procedure of CLSNet for CIL}
    \LinesNumbered
    \KwIn{Tasks $\bm{D}_1,\bm{D}_2,\dots,\bm{D}_T$ presented sequentially}
    \KwOut{Cloased-form solution $\bm{\varOmega}_{1:T}$ and declarative parameter statistic $\bm{\varOmega}$}  
    \For{$t=1, 2,\dots,T$}{
        \textbf{// Subnetwork I: Pre-Trained Feature Extractor}\\
        Initialize a pre-trained CNN model\;
        Obtain optimal representations $\bm{G}^{*}$ by Eqs. (\ref{Gi}-\ref{Theta})\;
        Set the output of Subnetwork I: $\bm{A}_t\leftarrow[\bm{Z}_t,\bm{G}^{*}]$\;
        \textbf{// Subnetwork II: Analytical SLFN Classifier}\\
        Compute declarative parameters $\bm{\varOmega}_t$ by Eq. (\ref{varOmega})\;
        \If{$t<T$}{
            Calculate the plasticity matrix $\bm{\mathcal{F}}_t$ by Eq. (\ref{myEFIM})\;
        }
        \For{$c=1, 2,\dots,C_t$}{
            Recursively consolidate $\bm{\varOmega}_{1:t,c}$ by Eq. (\ref{varOmega_cc})\;		
        }
        Obtain closed-form solution $\bm{\varOmega}_{1:t}=\{\bm{\varOmega}_{1:t,c}\}_{c=1}^{C_t}$\;
    } 
    \KwResult{CLSNet output $\bm{\widehat{Y}}_t=\bm{A}_t\bm{\varOmega}_{1:T}$ over tasks seen}
\end{algorithm}

\subsection{Managing the Trade-Off for Graceful Forgetting}
Throughout this paper, we use a non-growing architecture for CIL. It is natural to think of whether it has the opportunity to train on all incoming tasks given the limited network capacity. We argue that graceful forgetting is also a requisite against catastrophic forgetting, e.g., selectively removing inessential information is crucial to spare space for better learning the future ones. However, forgetting is subconscious and differs across people. Empirically, we simplify and investigate this issue by borrowing the concept of first-in-first-out (FIFO) \cite{barreiros2015queuing} in computing and systems theory. Concretely, we implement memory fading with a potential priority queue where the most previously learned tasks are the first to be removed, which is in line with the fact that the earlier experiences are easier to forget. CLSNet can innately manage the long/short-term declarative parameter of each seen task. This can be achieved by a minimal value in the trade-off $\lambda_t$ ($\gamma_t=\lambda_t N_{t+1}$). In this way, graceful forgetting of previously learned trivial tasks makes our method suffice to train on incoming tasks better.

\subsection{\textcolor{Blue}{Discussion}} \label{Discussion}
\textcolor{Blue}{
In this section, we discuss the relationship between the proposed CLSNet and the representative regularization-based method, EWC \cite{PNAS2017EWC}. Both serve the same purpose of mitigating catastrophic forgetting by accumulating quadratic terms to penalize changes in network parameters deemed important for old tasks. However, there are two intrinsic differences between CLSNet and EWC. (i) The original EWC forces a model to remember older tasks more vividly, which involves double counting the data and accumulating Fisher regularization at each layer~\cite{huszar2017quadratic}. This not only over-constrains the network parameters for learning new tasks but also leads to systematic bias favoring earlier tasks. Meanwhile, the entire EWC network experiences a linear growth in network parameters as the number of tasks increases. By contrast, our CLSNet only leverages the regularization in the single-hidden layer feed-forward network classifier, which builds on top of a plastic CNN feature extractor. This means that CLSNet only needs to maintain the Fisher information matrix in the decision layer, rather than the layer-wise implementation in EWC. Meanwhile, the systematic bias towards previously learned tasks can be effectively counterbalanced by incorporating the third term in Eq. (\ref{Obj_fun}), as demonstrated in the experiments (see Sec.~\ref{Ablation_5}). (ii) Rather than the commonly-used softmax layer, we derive a closed-form solution to prevent the decision boundary from being distorted, which is easy-implemented, parameter efficient, and can converge more quickly than a standard BP algorithm. Additionally, we find that some regularization-based methods~\cite{ICML2017SI, ECCV2018MAS}, including EWC, are highly susceptible to the trade-off coefficients for different settings (learning rate, batch size, etc.) and benchmarks (number of tasks and task orderings), which has been well tackled in our method. In particular, the closed-form solution has strong task-order robustness, with similar accuracies regardless of random task orderings for multiple runs (see Sec.~\ref{Ablation_4}).}

\section{Experiments} \label{Sec_4}
\subsection{Experiment Setup} \label{Experiment_setup}
\textbf{Datasets.}
We follow a popular data split in the CIL scenario to simulate emerging new classes by respectively disjointing three benchmark datasets, including FashionMNIST, CIFAR-100, and ImageNet-Subset \cite{tang2022learning, van2020replay, CVPR2021EFT, AAAI2021PCL}. For convenience, we use the nomenclature [DATASET]-$C/T$ to denote a task sequence with $C$ classes evenly divided into $T$ tasks, where the suffix indicates that a model needs to recognize $C/T$ new classes in each task (session).

\textbf{Evaluation metrics.}
We first adopt three main metrics to evaluate the CIL performance of different models in depth (all higher is better). \textbf{Avarage accuracy (Avg Acc)} measures the average test classification accuracy of a model on all tasks seen so far: $\text{Avg Acc}=\frac{1}{T}\sum_{t=1}^T R_{T,t}$, where $R_{T,t}$ is the test accuracy for task $t$ after training on task $T$; \textbf{Backward Transfer (BWT)} \cite{NIPS2017GEM} indicates a model's ability in knowledge retention, averaged over all tasks: $\text{BWT}=\frac{1}{T-1}\sum_{t=1}^{T-1} R_{T,t}-R_{t,t}$. Negative BWT means that learning new tasks causes forgetting past tasks; \textbf{Forward Transfer (FWT)} \cite{NIPS2020FROMP} measures how well a model uses previously learned knowledge to improve performance on the recently seen tasks: $\text{FWT}=\frac{1}{T-1}\sum_{t=2}^{T} R_{t,t}-R_t^{ind}$, where $R_t^{ind}$ is the test classification accuracy of an independent model trained only on task $t$. For similar ACC, the method that has a larger BWT/FWT is better. 

\textcolor{Blue}{
In addition to the above three accuracy-related metrics, we also report the indicators of training/computational efficiency. That is, the \textbf{Running Time (s)} and the \textbf{Memory Budget (MB)}. For the former, we measure the running time per epoch since different methods have very different requirements in computation; For the latter, we align the memory cost of both network parameters and old samples for fair comparisons~\cite{zhou2022model}, i.e., switching them to a 32-bit floating number. In this way, both the final model size (\#model, MB) and exemplar buffers (\#exemplar, MB) are counted into the memory budget (MB), calculated with an approximate summation of them.}


\textbf{Compared methods.}
We compare CLSNet with both classic and the latest baselines, which covers \textit{rehearsal-based approaches}: \textbf{FS-DGPM} \cite{deng2021flattening} \textbf{ARI} \cite{wang2022anti}, \textbf{GEM} \cite{NIPS2017GEM}, \textbf{RM} \cite{bang2021rainbow}, \textbf{RtF} \cite{van2018generative}, \textbf{BiR} \cite{van2020replay}, \textbf{IL2M} \cite{ICCV2019IL2M}, \textbf{LOGD} \cite{CVPR2021LOGD}; \textit{regularization-based approaches}: \textbf{EWC} \cite{PNAS2017EWC}, \textbf{SI} \cite{ICML2017SI}, \textbf{MAS} \cite{ECCV2018MAS}, \textbf{DMC} \cite{WACV2020DMC}, \textbf{OWM} \cite{NMI2019OWM}, \textcolor{Blue}{\textbf{InterContiNet} \cite{wolczyk2022continual}}; \textit{architecture-based approaches}: \textcolor{Blue}{\textbf{DER} \cite{yan2021dynamically}, \textbf{FOSTER}~\cite{wang2022foster}, \textbf{MEMO}~\cite{zhou2022model}, \textbf{DyTox}~\cite{douillard2022dytox}}, \textbf{EFT} \cite{CVPR2021EFT}, \textbf{RPS-Net} \cite{ NIPS2019RPS-Net}, \textbf{PCL} \cite{AAAI2021PCL}, and \textbf{SpaceNet} \cite{sokar2021spacenet}. Due to the peculiarity of ARI, InterContiNet, and DyTox, we do not compare them on the FashionMNIST-10/5. As baselines, we also compare with the naive approach of simply fine-tuning on each new task (\textbf{\textit{None}}; approximate lower bound) and a network that is always trained using the data of tasks seen so far (\textbf{\textit{Joint}}; approximate upper bound).

\textbf{Task protocols.} Only training samples of current task $t$ are available except for rehearsal-based methods, while test samples may come from any of tasks 1 to $t$ at inference time without knowing task identities. Instead of using a fixed task sequence, we run each benchmark five times with randomly shuffled task orders and then report the means and/or standard deviations of these results. Hence, the repeated multiple runs will be entered a sequence of tasks with different orders that are more practical in an open-ended environment. \textcolor{Blue}{We implement our method in PyTorch with NVIDIA RTX 3080-Ti GPUs, and the source code will be released.}

\subsection{Implementation Details} \label{Implementation_details}
\textbf{Network architecture.} In our experiments, all the methods use similar-sized neural network architectures for all the benchmarks, unless otherwise stated. For FashionMNIST-10/5, the architecture is empirically conducted through a simple multi-layer perceptron (MLP) with [784-900-900-10] neurons, among which we divide the penultimate layer into $n=30$ groups with each $k=30$ nodes for our method. No pre-trained CNN feature extractor is used for FashionMNIST-10/5 as a simple model already generates good results. For CIFAR-100, a standard ResNet-18 is employed to provide well-extracted features. We follow the setting in OWM \cite{NMI2019OWM, AAAI2021PCL} where a feature extractor pre-trained was first used to analyze the raw images. Then, the obtained feature vectors were sequentially fed into a simple model, with the similar architecture used in FashionMNIST-10/5, to learn the mappings between combinations of the features and labels of different classes. We apply the pre-trained feature extractor to the proposed CLSNet and all the baselines. \textcolor{Blue}{For the ImageNet-200, it takes the Tiny-ImageNet-200 as an auxiliary dataset to do pre-training with ResNet-50 and then starts CIL with another 200 classes that the pre-trained model is not previously encountered. Specifically, we use a competitive contrastive learning algorithm, SimSiam~\cite{chen2021exploring}, to initialize a pre-trained model. CLSNet first passes input images through the frozen SimSiam to obtain a set of rich representations; it then incrementally learns discriminative information which is unique to novel classes. This avoids data leakage as there are no overlapping classes between the data of pre-training and CIL. Note that initializing a contrastively pre-trained model is available for both our method and all the baselines for a fair comparison.}

\textbf{Hyper-parameter selection.} For all the baselines, we use the open-source code released by their authors. Specifically, we select the SGD optimizer with an initial learning rate of 0.1 for different datasets and the mini-batch size for each task is 100 (FashionMNIST and CIFAR-100) or \{100, 40, 20, 10\} (ImageNet-Subset). For the rehearsal-based methods, we restrict the exemplar memory budget to 2k samples by following the similar setting in RPS-Net \cite{NIPS2019RPS-Net}. By contrast, CLSNet keeps only one matrix ($900\times C$) per task or one vector ($900\times 1$) per class, which is highly efficient compared to storing samples, e.g., 0.03M (declarative parameter) and 5.98MB (rehearsal sample) memory budget for FashionMNIST. For the architecture-based methods, we report the model size after learning all tasks and the regularization branch determines the trade-off from the set \{100, 1000, 10000, 100000\}. Note that the other hyper-parameters are with reference to the original settings by default. For our method, the hyper-parameters used in our experiment are as follows: $\alpha = 0.01$, $\rho=2^{-30}$, and $\gamma_t=10000$ (see Sec.\ref{Ablation_1} in the ablation study for more).

\begin{table*}[htbp]
    \caption{\textcolor{Blue}{Results on FashionMNIST-10/5 measured by five evaluation metrics. The arrow “$\uparrow$” denotes the higher, the better; and the arrow“$\downarrow$” denotes the lower, the better. $^\dag$ refers to the results produced using a ResNet-18 backbone}}
    \label{Table_5tasks}
    \centering
    \begin{tabular}{llcccccc}
        \toprule
        \multirow{2}{*}{Catagory} &\multirow{2}{*}{Method} & \multirow{2}{*}{Avg Acc$\uparrow$} & \multirow{2}{*}{BWT$\uparrow$} & \multirow{2}{*}{FWT$\uparrow$} & \multirow{2}{*}{\textcolor{Blue}{Time (s)$\downarrow$}} & \multicolumn{2}{c}{\textcolor{Blue}{Memory budget}} \\
        & &    &    &    &   &\textcolor{Blue}{Model (MB)$\downarrow$}   & \textcolor{Blue}{Exemplar (MB)$\downarrow$}    \\ \midrule
        \multirow{2}{*}{Non-CL} &\textit{None} (lower bound)   &$\sim$0.1981  &-  &-   &50.45  &5.83  &\ding{55}  \\
        & \textit{Joint}~ (upper bound) &$\sim$0.9661  &-  &-   &59.01  &5.83  &\ding{55}  \\ \midrule
        
        \multirow{5}{*}{Regularization-based} &MAS \cite{ECCV2018MAS} &0.3344 &$-$0.5897 &$-$0.2092 &102.73 &17.47 &\ding{55}\\
        &EWC \cite{PNAS2017EWC} &0.4041 &$-$0.7474 &$-$0.0737 &80.54 &52.43 &\ding{55}\\
        &SI \cite{ICML2017SI}   &0.5567 &$-$0.3748 &$-$0.1325 &66.28 &17.47 &\ding{55}\\
        &DMC \cite{WACV2020DMC} &0.7103 &$-$0.3510 &$-$0.2647 &43.21 &11.65 &\ding{55}\\
        &OWM \cite{NMI2019OWM} &0.8012 &$-$0.1658 &$-$0.0713 &49.23 &11.65 &\ding{55}\\\midrule
        
        \multirow{7}{*}{Rehearsal-based} 
        &BiR \cite{van2020replay} &0.6928 &$-$0.3645 &$-$0.0208 &29.38 &5.83 &5.98\\
        &RtF \cite{van2018generative} &0.7107 &$-$0.2951 &$-$0.0296 &41.23 &5.83  &5.98\\
        &FS-DGPM \cite{deng2021flattening}  &0.8058 &$-$0.1358 &$-$0.0930 &66.10 &5.83 &5.98\\
        &GEM \cite{NIPS2017GEM}  &0.8299 &$-$0.0714 &$-$0.0901 &50.85 &5.83 &5.98\\
        &LOGD \cite{CVPR2021LOGD} &0.8302 &$-$0.1105 &$-$0.0498 &340.18 &5.83 &5.98\\
        &IL2M \cite{ICCV2019IL2M} &0.8703 &$-$0.0683 &$-$0.0201 &53.56 &5.83 &5.98 \\ 
        &RM \cite{bang2021rainbow} &0.8875 &$-$0.1632 &$-$0.0748 &188.09 &5.83 &5.98\\ \midrule
        
        \multirow{7}{*}{Architecture-based} &SpaceNet \cite{sokar2021spacenet} &0.6481 &$-$0.2102 &$-$0.0628 &14.97 &5.83 &\ding{55}\\
        &EFT$^\dag$ \cite{CVPR2021EFT} &0.7953 &$-$0.1528 &$-$0.0652 &152.46 &42.61 &\ding{55} \\
        &RPS-Net \cite{NIPS2019RPS-Net} &0.8061 &$-$0.0297 &$-$0.0305 &59.81 &46.64 &5.98\\
        &PCL \cite{AAAI2021PCL}  &0.8125 &$-$0.1305 &$-$0.0679 &430.25 &5.83 &\ding{55} \\
        &\textcolor{Blue}{DER~$^\dag$ \cite{yan2021dynamically}}  &0.8632 &$-$0.1371 &$-$0.0224 &145.98 &213.60 &5.98\\ 
        &\textcolor{Blue}{MEMO$^\dag$~\cite{zhou2022model}}  &0.8758 &$-$0.1386 &$-$0.0259 &30.48 &138.76 &5.98\\
        &\textcolor{Blue}{FOSTER$^\dag$~\cite{wang2022foster}}  &0.8831 &$-$0.1640 &$-$0.0985 &41.42 &85.22 &5.98\\ \midrule
        
        \multirow{1}{*}{Ours} &CLSNet  &0.9185  &$-$0.0644  &$-$0.0194 &15.89  &6.13 &\ding{55}\\ \bottomrule
    \end{tabular}
\end{table*}

\subsection{Results and Discussion} \label{Results_and_discussion}

\subsubsection{Results on FashionMNIST-10/5}
The task sequence is formed by splitting ten objects into five two-class classification tasks, among which each one is sequentially presented. Our method exhibits competitive superiority in the involved five evaluation metrics, as reported in Table \ref{Table_5tasks}. \textit{ First}, we improve Avg Acc upon the second-best method RM by an absolute margin of 3.10\%, concurrently with an acceptable level in the classification error rates compared to the upper bound. We note that rehearsal-based methods obtain better results in general as they can retrieve previously learned knowledge from the exemplar buffers. And among rehearsal-free methods, dynamic networks struggle in the CIL scenario while non-growing networks with regularization suffer from catastrophic forgetting to different degrees. \textcolor{Blue}{\textit{Second}, BWT and FWF of ours come out in front, which is slightly inferior to the highest BWT achieved by RPS-Net but shares the highest FWT. These imply that the proposed method can well transfer knowledge across tasks: (i) the frozen pre-trained feature extractor together with diverse representation augmentation is beneficial to BWT (see the middle panel of Fig. \ref{Fig_CLSNet}), and (ii) the declarative parameter of each old task is well consolidated in the decision layer for FWT (see the right panel of Fig. \ref{Fig_CLSNet}). \textit{Third}, we also show the computational comparison of different methods to indicate the training efficiency. For the running time, our method is remarkably superior to others due to the advantage of closed-form solutions, e.g., faster convergence; The final model size of our method is slightly ($\sim$5\%) larger than the rehearsal-based ones but our method needs no exemplar buffers. These results on FashionMNIST-10/5 demonstrate that the proposed CLSNet is a parameter-efficient CIL method.}

\subsubsection{Results on CIFAR-\{100/5, 100/10\}}
Fig. \ref{CIFAR100_Split5} further breaks down the 100 classes into 10 tasks and compares different methods on the resulting CIFAR-100/10 task sequence. In general, both architecture- and rehearsal-based methods exhibit superior accuracy compared with regularization-based ones. It is observed that the results indicate the competitive performance of our CLSNet in each CIL session, showing consistency with that of FashionMNIST-10/5. Compared with the last accuracy, our method achieves a relative gain of 1.94\% over the second-best method.  We note that our performance is only 8.07\% below the \textit{Joint} approach that is offline trained by using the samples of all classes. By contrast, the \textit{None} approach fails to classify all test images into the classes of the most recently learned task.

\begin{figure}[htbp]
    \centering
    \includegraphics[height=3.2in]{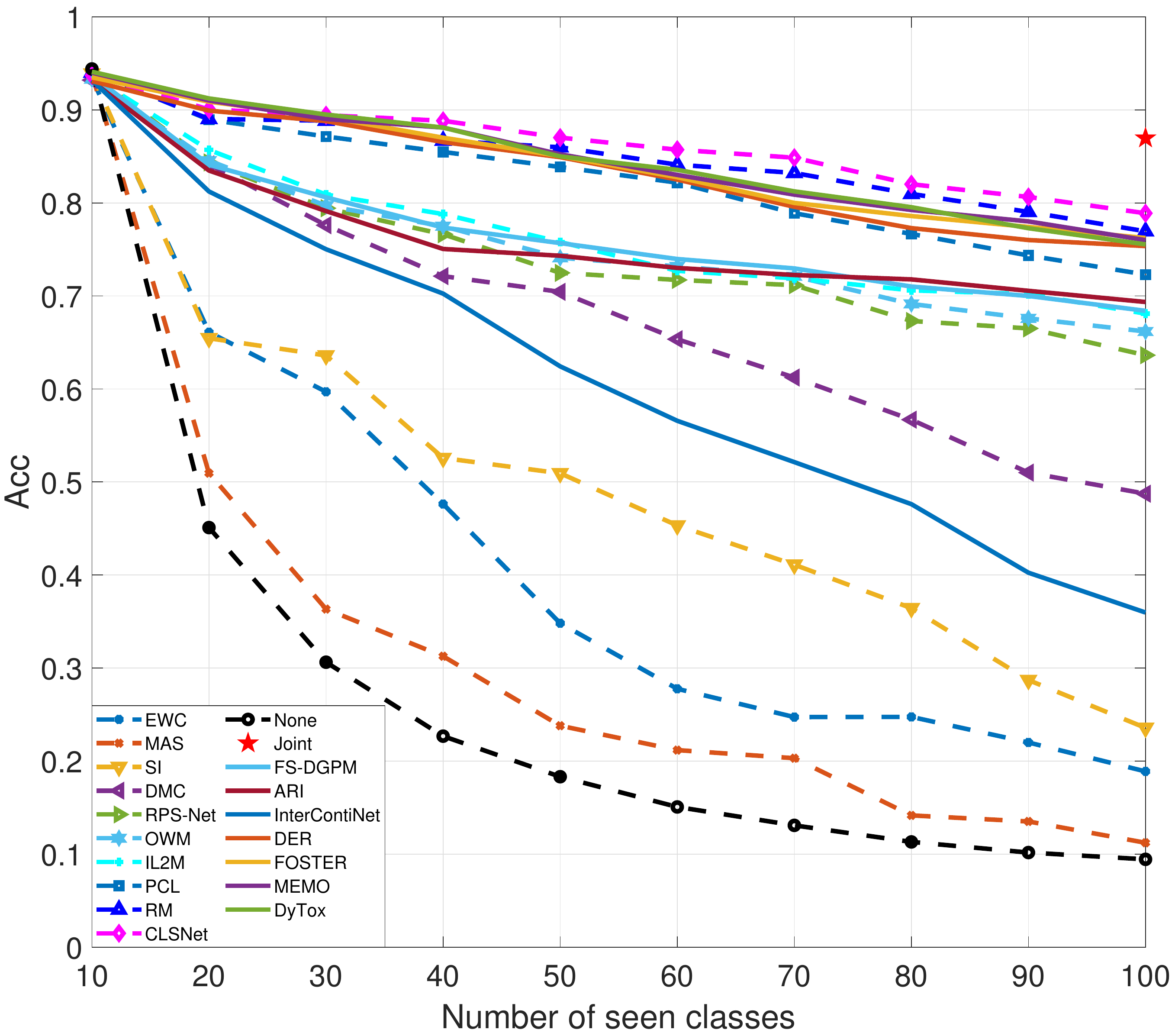}
    \caption{\textcolor{Blue}{Results on CIFAR-100/10, where each method needs to incrementally learn 10 new classes per task.}}
    \label{CIFAR100_Split5}
\end{figure}

\subsubsection{Results on ImageNet-\{200/5, 200/10, 200/20, 200/40\}}
More challenging ImageNet-based task sequences are reported in Table \ref{Table_2}, where we incrementally learn groups of 5, 10, 20, and 50 classes in each training session. Among them, ImageNet-200/5 with more classes per task requires the model to learn a harder problem for each task, while ImageNet-200/40 increases the length of the task sequence, testing a CL method’s knowledge retention. The results indicate superior performance of the proposed method in all cases. \textcolor{Blue}{Specifically, we observe that the proposed method surpasses all the selected SOTA baselines for the ImageNet-200/10 and ImageNet-200/40 task sequences. Compared with the second-best baseline, our method improve it by a margin of 0.73\% and 0.16\% respectively. For the ImageNet-200/5 and ImageNet-200/20 task sequences, our method is slightly inferior to the best method. The comparison with the SOTA methods highlights our incremental learner as a promising tool for mitigating catastrophic forgetting.}

\begin{table}[htbp]
    \caption{\textcolor{Blue}{Comparison of our method with the SOTA baselines on ImageNet-Subset task sequences measured by Avg Acc. We mark the best results in \textbf{Bold} and second-best results in \textbf{\textit{Italic}}}}
    \label{Table_2}
    \centering
    \begin{tabular}{lcccc}
        \toprule
        \multirow{2}{*}{Method} & \multicolumn{4}{c}{ImageNet-} \\ \cmidrule{2-5}
        &200/5 & 200/10 & 200/20 & 200/40 \\ \midrule
        OWM \cite{NMI2019OWM}  &0.5315 &0.5498  &0.5169 & 0.4953  \\
        IL2M \cite{ICCV2019IL2M} &0.5413 &0.5401  &0.5047 & 0.4897  \\
        PCL \cite{AAAI2021PCL}  &0.5342 &\textbf{\textit{0.5642}}  &0.5247 & 0.5217\\
        RM \cite{bang2021rainbow} &0.5398 &0.5569  &0.5279 & 0.5103    \\
        \textcolor{Blue}{FS-DGPM \cite{deng2021flattening}}  &0.5314 &0.5499  &0.5337 & 0.5241    \\
        \textcolor{Blue}{ARI \cite{wang2022anti}} &0.5464 &0.5285  &0.4946 & 0.5017    \\     
        \textcolor{Blue}{FOSTER \cite{wang2022foster}} &0.5461 &0.5601  &0.5327 & \textbf{\textit{0.5258}}   \\ 
        \textcolor{Blue}{DyTox \cite{douillard2022dytox}}  &\textbf{0.5498} &0.5623  &\textbf{0.5412} & 0.5231    \\ 
    
        \midrule
        Ours &\textit{\textbf{0.5468}} &\textbf{0.5715}  &\textit{\textbf{0.5383}} & \textbf{0.5274} \\
        \bottomrule
    \end{tabular}
\end{table}

\subsection{Ablation Studies} \label{Ablation_study}
\subsubsection{Effectiveness of Diverse Representation Augmentation}
\label{Ablation_2}
As described in Sec. \ref{Subnetwork_I}, $g^{\prime}(\bm{X}_t; \bm{\theta}_{g^{\prime}})$ are expanded with $g^{\prime\prime}(\bm{Z}_t; \bm{\theta}^{*}_{g^{\prime\prime}})$ as a whole to augment the transferability on a sequence of tasks (see Fig. \ref{Fig_CLSNet}). Here we use four different types of connections to validate the effectiveness of representation augmentation: (a) $\bm{Z}_t$ means the output of $g^{\prime}(\bm{X}_t; \bm{\theta}_{g^{\prime}})$---drifted representations; (b) $\bm{G}$ means the output of $g^{\prime\prime}(\bm{Z}_t; \bm{\theta}_{g^{\prime\prime}})$---fixed representations; (c) $\bm{G}^{*}$ means the output of $g^{\prime\prime}(\bm{Z}_t; \bm{\theta}^{*}_{g^{\prime\prime}})$---optimal representations; and (d) $\bm{A}_t=[\bm{Z}_t,\bm{G}^{*}]$ means the augmented representations used in this paper. Table \ref{Table_Ablation} reports the Avg Acc (\%) of both task 1 and all tasks, respectively. It can be seen that all of them are effective except for the case of $\bm{Z}_t$. Specifically, the performance of $\bm{A}_t$ on the first task is slightly inferior to $\bm{G}^{*}$. However, it is highly effective in terms of the Avg Acc on all five tasks. This is because the representations obtained by a frozen pre-trained model may not be sufficiently discriminative on a sequence of tasks.

\begin{table}[htbp]
    \caption{Experiments under different connections in the proposed CLSNet reported by the Avg Acc (\%) of task 1 \& all}
    \label{Table_Ablation}
    \centering
    \begin{tabular}{ccc}
        \toprule
        Connection & CIFAR-100/10 & CIFAR-100/20 \\ \midrule
        $\bm{Z}_t$ &89.90 \& 61.25   &76.80 \& 11.92  \\
        $\bm{G}^{~}$ &92.17 \& 63.85  &92.32 \& 57.86  \\
        $\bm{G}^{*}$ &\textbf{92.78} \& 77.23  & \textbf{93.25} \& 73.58 \\
        $\bm{A}_t$ &92.56 \& \textbf{78.43}  &93.17 \& \textbf{74.65} \\
        \bottomrule
    \end{tabular}
\end{table}

\subsubsection{Visualization of the Optimal Representations}
\label{Ablation_3}
To further demonstrate the effectiveness of augmenting transferability, we visualize the optimal representations which will be used for the final classifier. Take FashionMNIST-10/5 as an example, in which 10 classes are evenly split into five 2-class tasks. Fig. \ref{Visualization} shows the t-SNE visualization \cite{van2008visualizing} of the raw sample space $\bm{X}_t$ and the optimal representations $\bm{A}_t$ ($t=1,2,\dots,5)$ of each class. It can be observed that the same classes were well clustered while different classes are clearly separated. Therefore, the optimal representations produced by the plastic pre-trained model could provide sufficiently distinguishable information from a sequence of tasks for the final decision. 

\begin{figure}[tbp]
    \centering
    \subfloat[]{\includegraphics[height=0.85in]{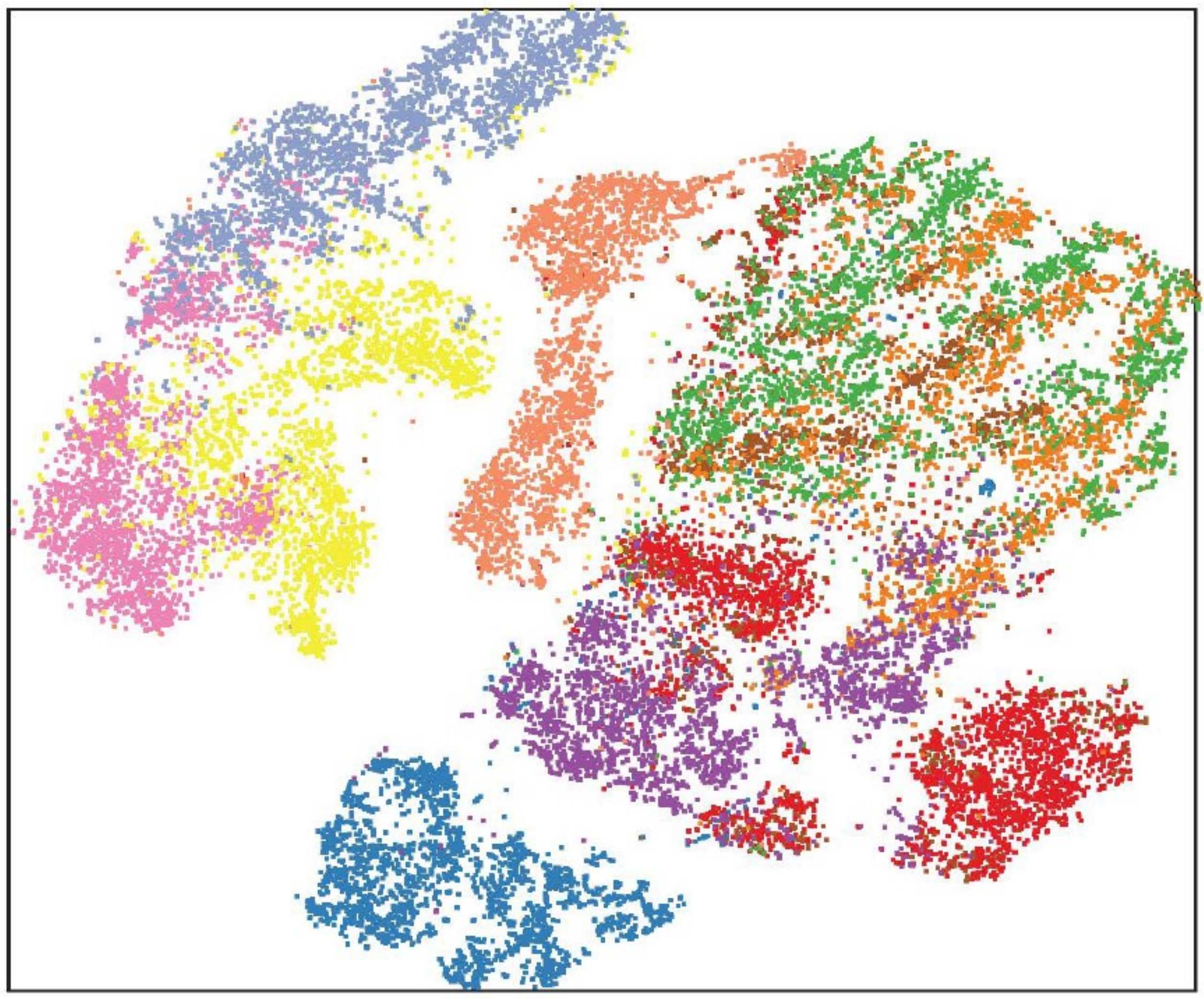}
        \label{Raw_split5}}
    \hfil 
    \subfloat[]{\includegraphics[height=0.85in]{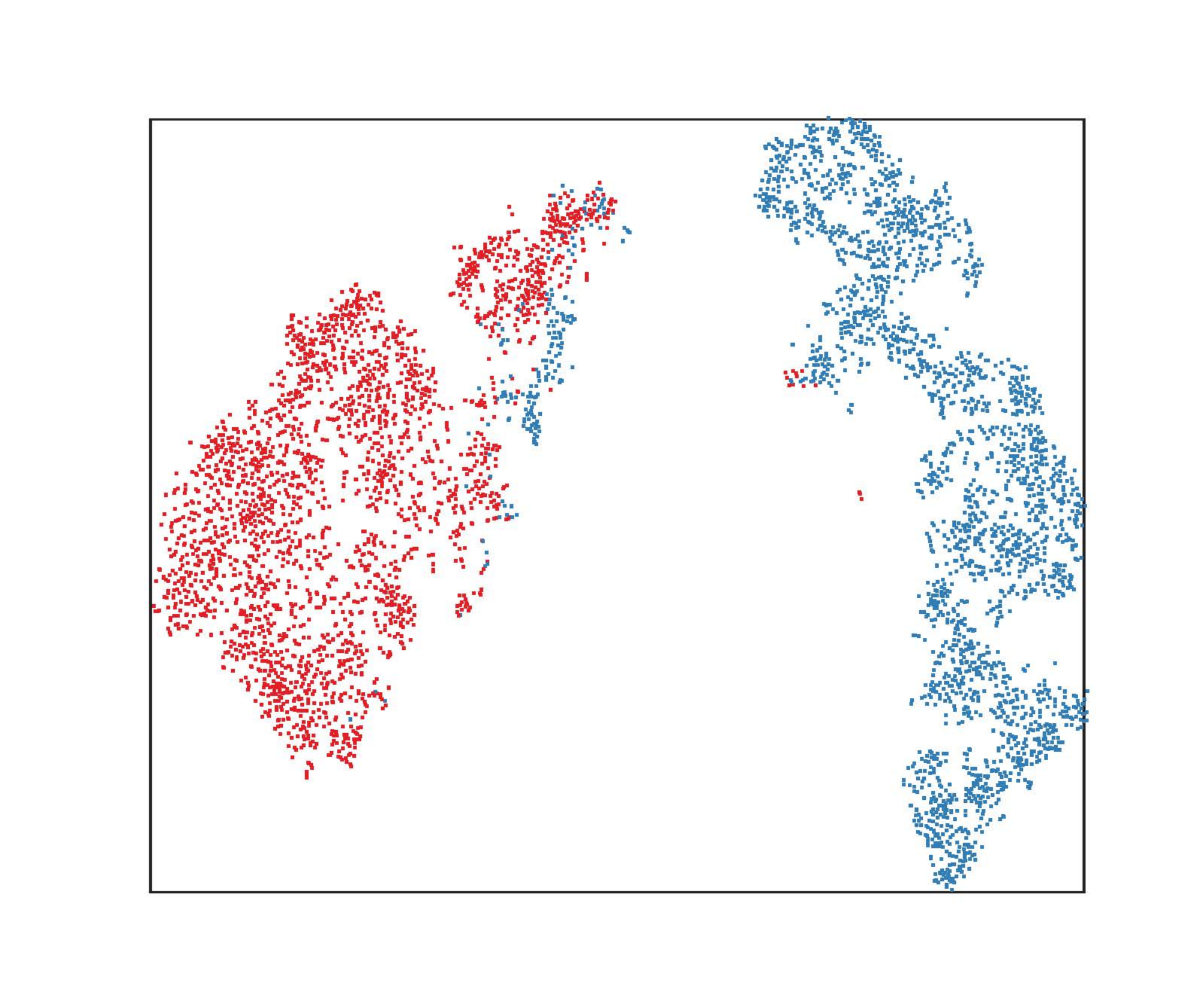}
        \label{Task1}}
    \hfil
    \subfloat[]{\includegraphics[height=0.85in]{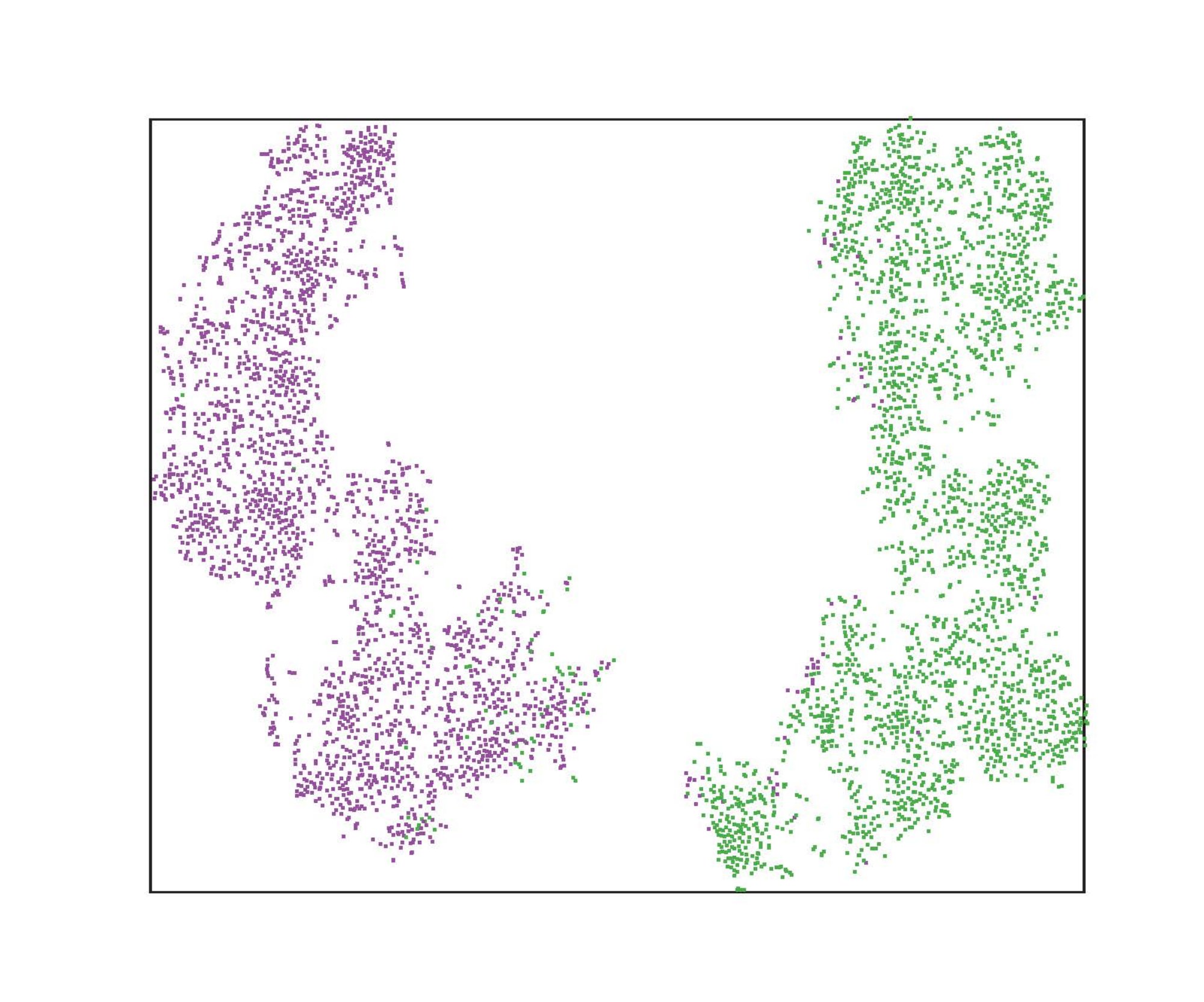}
        \label{Task2}}
    \hfil
    \subfloat[]{\includegraphics[height=0.85in]{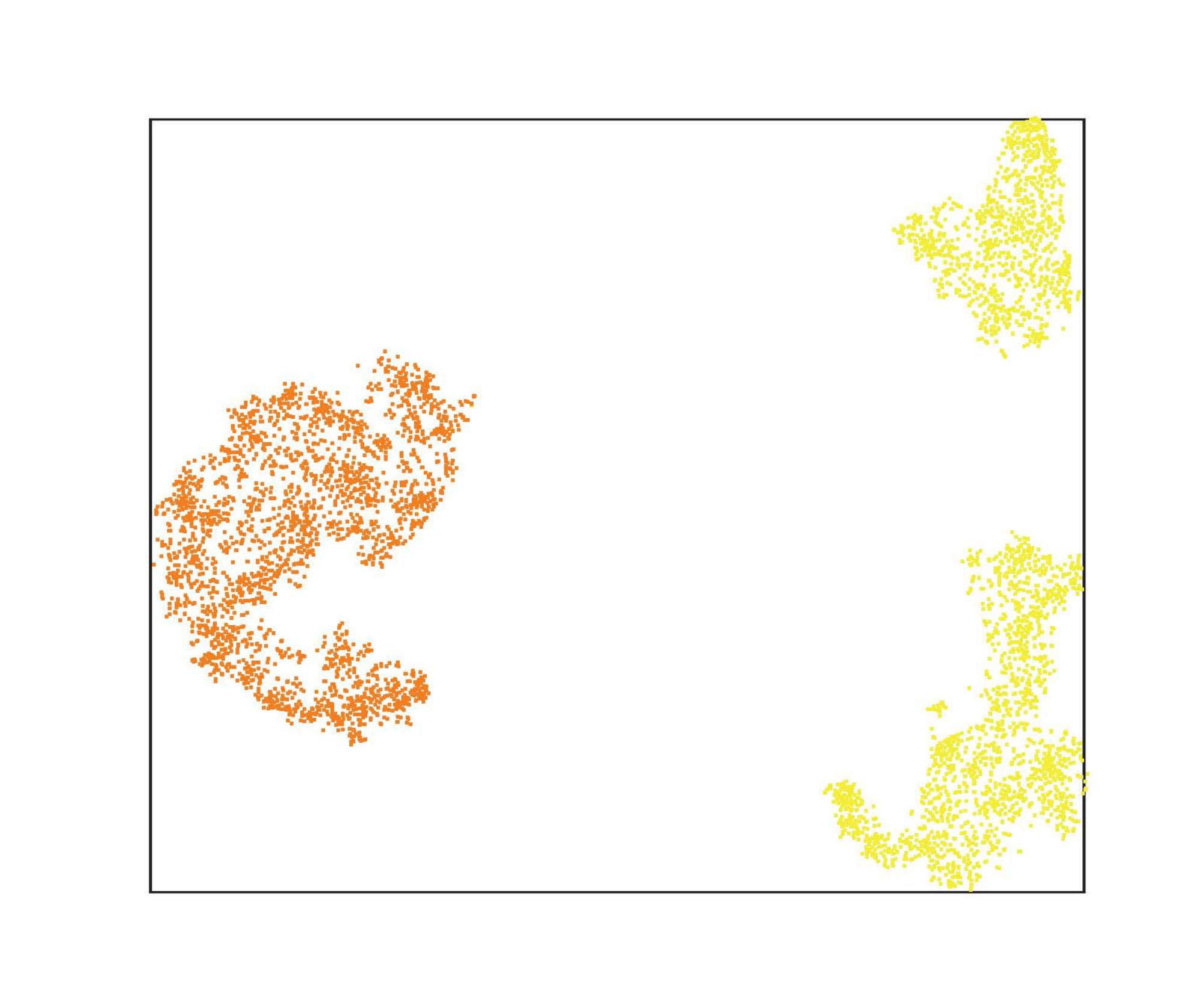}
        \label{Task3}}
    \hfil
    \subfloat[]{\includegraphics[height=0.85in]{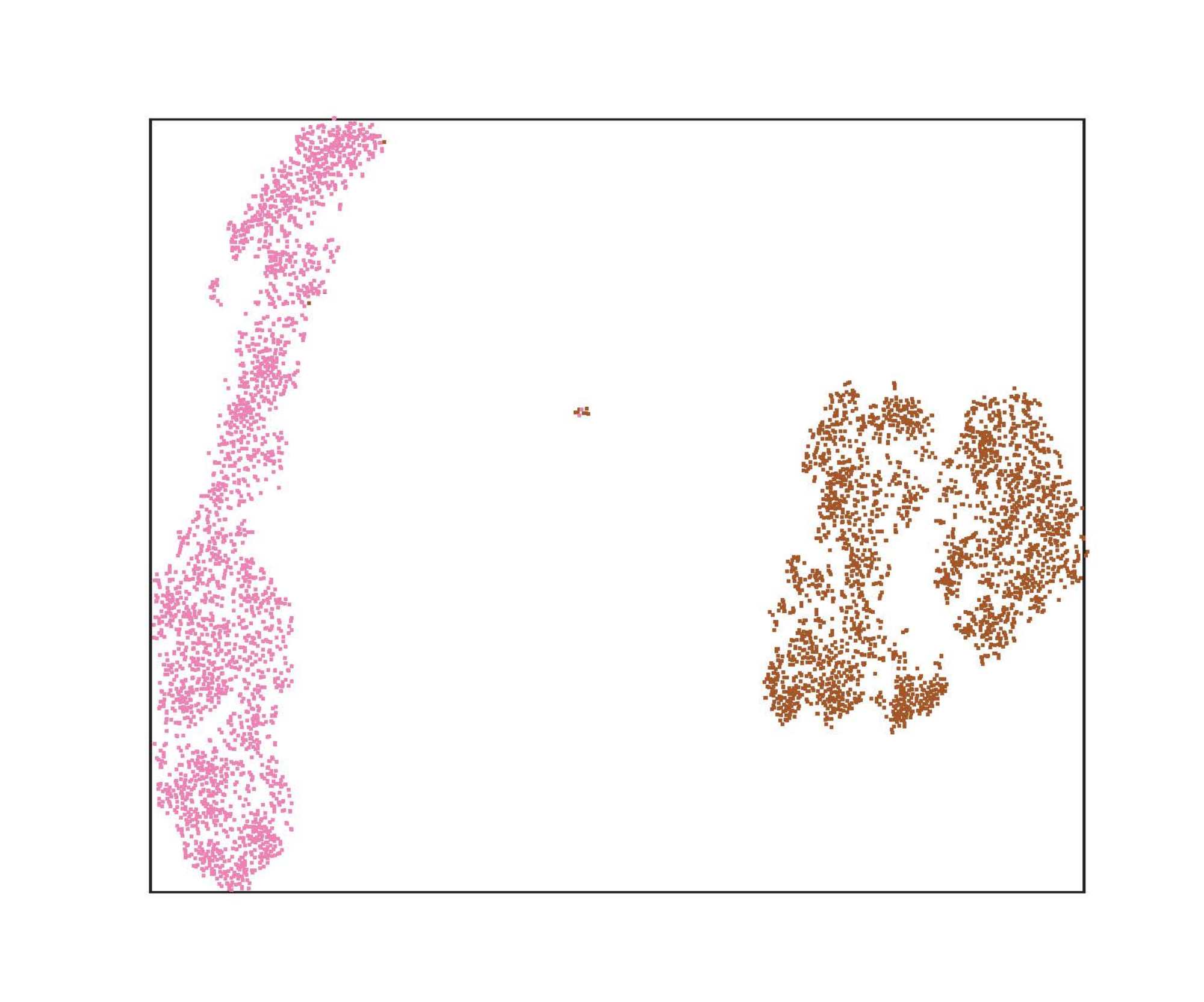}
        \label{Task4}}
    \hfil
    \subfloat[]{\includegraphics[height=0.85in]{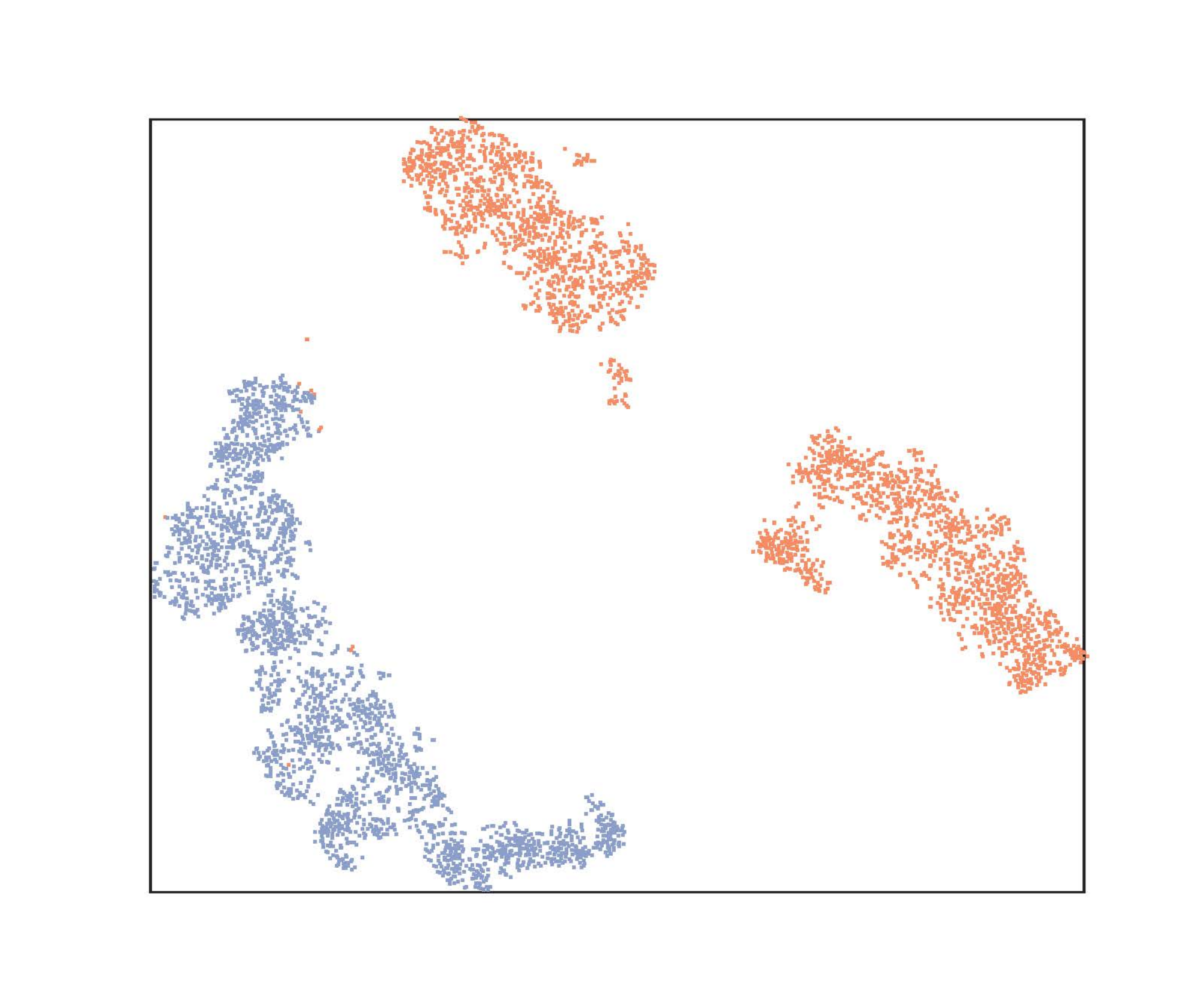}
        \label{Task5}}
    \caption{t-SNE visualization based on FashionMNIST-10/5. Each color represents a class. In this task sequence, training data of all the classes together are first visualized as a reference, followed by the optimal representations after incrementally learning two classes per session. (a) Mixed raw sample space. (b)-(f) Well-clustered representation space.}
    \label{Visualization}
\end{figure}

\subsubsection{\textcolor{Blue}{Effectiveness of Each Component in Objective Function}}
\label{Ablation_5}
\textcolor{Blue}{We then provide an empirical investigation of the effectiveness of each element in Eq. (\ref{Obj_fun}). For simplicity, we denote the corresponding three components as Term 1, Term 2, and Term 3 respectively. Taking FashionMNIST as an example, Table~\ref{Table_Terms} shows the experimental results of our method under different components. (a) In the case of using only Term 1, CLSNet degenerates into a traditional algorithm that is made available for single-task learning. (b) In the case of using Terms 1 and 2, the catastrophic forgetting of old tasks is mitigated as the additional regularization term could constrain the parameters not to deviate too much from those that were optimized. (c) In the case of using Terms 1 and 3, there is a serious performance drop. This is because it fails to globally treat parameters in the decision layer as having a different level of importance. (d) Finally, the last row demonstrates the collaborative performance of all components. Compared with case (b), the marginal boost of case (d) can be closely related to the counterbalanced systematic bias towards old tasks.}

\begin{table}[htbp]
    \caption{\textcolor{Blue}{Experiments under different components in the proposed CLSNet via the FashionMNIST-10/5 task sequence\\ measured by Avg Acc (\%), BWT and FWT}}
    \label{Table_Terms}
    \centering
    \begin{tabular}{cccccc}
        \toprule
        \multicolumn{3}{c}{Component} & \multicolumn{3}{c}{Metric} \\ 
        Term 1   & Term 2   & Term 3   &Avg Acc(\%)   & BWT   & FWT  \\ \midrule
        \ding{51} &\ding{55} &\ding{55}       &19.64 &-0.9903 &0.0004   \\
        \ding{51} &\ding{51} &\ding{55}       &90.43 &-0.0751 &-0.0299   \\
        \ding{51} &\ding{55} &\ding{51}       &25.70  &-0.8954 &-0.0187  \\
        \ding{51} &\ding{51} &\ding{51}       &91.69 &-0.0650 &-0.0203  \\
        \bottomrule
    \end{tabular}
\end{table}

\begin{table}[htbp]
    \caption{Experiments on the task-order robustness reported with the means and standard deviations under multiple random runs}
    \label{Task_order}
    \centering
    \begin{tabular}{lcc}
        \toprule
        Method & FashionMNIST-10/5 & CIFAR-100/5 \\ \midrule
        EWC &42.90$\pm$2.35  &25.14$\pm$3.94   \\
        SI &56.27$\pm$4.02  &35.53$\pm$2.85  \\
        OWM &79.87$\pm$1.25  &66.49$\pm$0.85  \\
        Ours &91.48$\pm$\textbf{0.40}  &76.54$\pm$\textbf{0.73} \\
        \bottomrule
    \end{tabular}
\end{table}

\begin{table*}[tbp]
    \caption{Results on different settings of trade-off via the CIFAR-100/10 task sequence measured by Avg Acc (\%), the final BWT and FWT}
    \label{Trade-off}
    \centering
    \begin{tabular}{lccccccccccrr}
        \toprule
        \multicolumn{1}{c}{\multirow{2}{*}{\begin{tabular}[c]{@{}l@{}}Trade-off \end{tabular}}} & \multicolumn{10}{c}{Avg Acc (\%)} &  &  \\ \cmidrule{2-11}
        & 1 & 2 & 3 & 4 & 5 & 6 & 7 & 8 & 9 & 10 & BWT & FWT \\ \midrule
        1 &93.70 &45.00 &36.37 &28.88 &22.42 &24.98 &19.26 &16.15 &14.23 &16.91 &$-$0.8192 &\textbf{0.0003}  \\
        $10^2$ &93.70 &86.95 &87.87 &86.50 &84.60 &83.12 &81.34 &79.52 &77.38 &74.95 &$-$0.0781 &$-$0.0008 \\
        $10^4$ &93.70 &\textbf{90.05} &\textbf{89.43} &\textbf{88.82} &\textbf{87.02} &\textbf{85.42} &\textbf{84.56} &\textbf{83.61} &\textbf{81.58} &\textbf{78.12}           &$-$0.0117 &$-$0.0375  \\
        $10^6$ &93.70 &84.45 &79.87 &75.95 &74.22 &73.48 &71.21 &70.45 &69.13 &67.58 &\textbf{0.4712} &$-$0.6547   \\ \bottomrule
    \end{tabular}
\end{table*}

\subsubsection{Task-Order Robustness of Analytical Solution}
\label{Ablation_4}
As the task/class orders are unknown in advance for CIL, Table \ref{Task_order} explores the robustness of regularization-based methods under randomly reshuffled task orders. For example, a task sequence with 5 tasks can generate 120 different combinations of orders. Here, we report both the means and standard deviations by running each benchmark five times with randomly shuffled task orders. We observe that the selected baselines are vulnerable to task-order robustness while our method with the analytical solution has the lowest sensitivity, followed by OWM \cite{NMI2019OWM}. Interestingly, OWM is actually an approximation of the analytical solution via the gradient descent method.

\subsubsection{Analysis on the Trade-Off}
\label{Ablation_1}
Table \ref{Trade-off} reports the results of the proposed CLSNet under different trade-off $\gamma_t$ values, which controls how important old tasks are compared to the current task. \textit{First}, our method with $\gamma_t=1$ degenerates into the \textit{None} approach such that the learning is almost completely biased towards the current task and yields positive FWT values. This implies that it would sacrifice previously learned knowledge to improve performance on the newly encountered tasks. \textit{Second}, $\gamma_t=10^2$ indicates that the plasticity on old tasks is very limited and the focus is still on the current task. \textit{Third}, the $\gamma_t=10^4$ setting is used in our method. \textit{Finally}, interference among multiple tasks can cause great damage to the new task when $\gamma_t=10^6$, which obtains a positive BWT value at the cost of current tasks.

\begin{table}[tbp]
    \setlength{\tabcolsep}{1.0mm} 
    \caption{Results on graceful forgetting via the FashionMNIST-10/5 task sequence. We also report the \underline{Avg Acc} (\%) marked by underline on all seen tasks so far}
    \label{Table_disposable}
    \centering
    \begin{tabular}{lcclcccccc}
        \toprule
        \multirow{2}{*}{Metric} & \multicolumn{2}{c}{\multirow{2}{*}{LTM}}&\multirow{2}{*}{} & \multicolumn{6}{c}{STM} \\ \cmidrule(r){5-10}
        & \multicolumn{3}{c}{} & \multicolumn{2}{c}{task 1} & \multicolumn{2}{c}{task 2} & \multicolumn{2}{c}{tasks 1 \& 2} \\ \midrule
        $R_{1,1}$ &96.25 &\underline{96.25}& &96.25 &\underline{96.25} &96.25 &\underline{96.25} & 96.25 &\underline{96.25}\\ \cmidrule(r){1-9}
        $R_{2,1}$ &91.75  &\multicolumn{1}{c}{\multirow{2}{*}{\underline{93.98}}} & &-  &\multicolumn{1}{c}{\multirow{2}{*}{\underline{98.50}}} &91.75  &\multicolumn{1}{c}{\multirow{2}{*}{\underline{93.98}}} &-  &\multicolumn{1}{c}{\multirow{2}{*}{\underline{98.50}}}    \\
        $R_{2,2}$ &96.20  & & &98.50  &  &96.20  & &98.50  &    \\ \cmidrule(r){1-10}
        $R_{3,1}$ &92.20  &\multicolumn{1}{c}{\multirow{3}{*}{\underline{93.18}}}& &-  &\multicolumn{1}{c}{\multirow{3}{*}{\underline{95.34}}}  &91.00  &\multicolumn{1}{c}{\multirow{3}{*}{\underline{95.28}}}  &-  &\multicolumn{1}{c}{\multirow{3}{*}{\underline{96.65}}}    \\
        $R_{3,2}$ &88.65  & & &91.05  &  &-  &  &-  &    \\
        $R_{3,3}$ &98.70  & & &99.60  &  &99.55  &  &96.65  &    \\ \cmidrule(r){1-10}
        $R_{4,1}$ &92.00  &\multicolumn{1}{c}{\multirow{4}{*}{\underline{91.84}}} & &-  &\multicolumn{1}{c}{\multirow{4}{*}{\underline{96.50}}}  &92.35  &\multicolumn{1}{c}{\multirow{4}{*}{\underline{96.25}}}  &-  &\multicolumn{1}{c}{\multirow{4}{*}{\underline{98.80}}}    \\
        $R_{4,2}$ &84.10  & & &93.40  &  &-  &  &-  &    \\
        $R_{4,3}$ &95.45  & & &97.05  &  &97.55  &  &98.15  &    \\
        $R_{4,4}$ &95.80  & & &99.05  &  &98.55  &  &99.45  &    \\ \cmidrule(r){1-10}
        $R_{5,1}$ &98.25  &\multicolumn{1}{c}{\multirow{5}{*}{\underline{91.05}}} & &-  &\multicolumn{1}{c}{\multirow{5}{*}{\underline{94.80}}} &89.95  &\multicolumn{1}{c}{\multirow{5}{*}{\underline{94.78}}}  &-  &\multicolumn{1}{c}{\multirow{5}{*}{\underline{96.65}}}    \\
        $R_{5,2}$ &78.00  & & &90.70  &  &-  &  &-  &    \\
        $R_{5,3}$ &96.30  & & &95.30  &  &95.25  &  &95.10  &    \\
        $R_{5,4}$ &96.35  & & &96.10  &  &97.05  &  &97.05  &    \\
        $R_{5,5}$ &96.00  & & &97.10  &  &96.90  &  &97.80  &   \\ \bottomrule
    \end{tabular}
\end{table}

\begin{table}[tbp]
    \setlength{\tabcolsep}{1.0mm} 
    \caption{Results on graceful forgetting via the CIFAR-100/5 task sequence. We also report the \underline{Avg Acc} (\%) marked by underline on all seen tasks so far}
    \label{Table_disposable2}
    \centering
    \begin{tabular}{lcclcccccc}
        \toprule
        \multirow{2}{*}{Metric} & \multicolumn{2}{c}{\multirow{2}{*}{LTM}}&\multirow{2}{*}{} & \multicolumn{6}{c}{STM} \\ \cmidrule(r){5-10}
        & \multicolumn{3}{c}{} & \multicolumn{2}{c}{task 1} & \multicolumn{2}{c}{task 2} & \multicolumn{2}{c}{tasks 1 \& 2} \\ \midrule
        $R_{1,1}$ &86.25 &\underline{86.25}& &86.25 &\underline{86.25} &86.25 &\underline{86.25} & 86.25 &\underline{86.25}\\ \cmidrule(r){1-9}
        $R_{2,1}$ &82.25  &\multicolumn{1}{c}{\multirow{2}{*}{\underline{81.83}}} & &-  &\multicolumn{1}{c}{\multirow{2}{*}{\underline{82.40}}} &82.25  &\multicolumn{1}{c}{\multirow{2}{*}{\underline{81.83}}} &-  &\multicolumn{1}{c}{\multirow{2}{*}{\underline{82.40}}}    \\
        $R_{2,2}$ &81.40  & & &82.40  &  &81.40  & &82.40  &    \\ \cmidrule(r){1-10}
        $R_{3,1}$ &79.90  &\multicolumn{1}{c}{\multirow{3}{*}{\underline{80.83}}} &&-  &\multicolumn{1}{c}{\multirow{3}{*}{\underline{81.13}}}  &80.90  &\multicolumn{1}{c}{\multirow{3}{*}{\underline{82.75}}}  &-  &\multicolumn{1}{c}{\multirow{3}{*}{\underline{85.00}}}    \\
        $R_{3,2}$ &78.40  & & &78.20  &  &-  &  &-  &    \\
        $R_{3,3}$ &84.20  & & &84.05  &  &84.60  &  &85.00  &    \\ \cmidrule(r){1-10}
        $R_{4,1}$ &79.50  &\multicolumn{1}{c}{\multirow{4}{*}{\underline{80.18}}} & &-  &\multicolumn{1}{c}{\multirow{4}{*}{\underline{80.78}}}  &79.95  &\multicolumn{1}{c}{\multirow{4}{*}{\underline{81.82}}}  &-  &\multicolumn{1}{c}{\multirow{4}{*}{\underline{83.02}}}    \\
        $R_{4,2}$ &76.75  & & &77.10  &  &-  &  &-  &    \\
        $R_{4,3}$ &80.65  & & &81.10  &  &81.25  &  &81.40  &    \\
        $R_{4,4}$ &83.80  & & &84.15  &  &84.25  &  &84.65  &    \\ \cmidrule(r){1-10}
        $R_{5,1}$ &77.35  &\multicolumn{1}{c}{\multirow{5}{*}{\underline{78.94}}}&  &-  &\multicolumn{1}{c}{\multirow{5}{*}{\underline{80.25}}} &77.50  &\multicolumn{1}{c}{\multirow{5}{*}{\underline{79.43}}}  &- &\multicolumn{1}{c}{\multirow{5}{*}{\underline{82.28}}}    \\
        $R_{5,2}$ &76.95  & & &77.35  &  &-  &  &-  &    \\
        $R_{5,3}$ &79.90  & & &80.55  &  &80.90  &  &81.25  &    \\
        $R_{5,4}$ &75.55  & & &77.55  &  &75.15  &  &80.10  &    \\
        $R_{5,5}$ &84.95  & & &85.55  &  &84.25  &  &85.50  &   \\ \bottomrule
    \end{tabular}
\end{table}

\subsection{Investigation on Graceful Forgetting}
Before concluding our work, we investigate the necessity of graceful forgetting given a non-growing backbone. A system with bounded network capacity that retains memories over an entire lifetime will have very little margin for new experiences. That is, it would eventually run out of capacity to learn incoming tasks. Therefore, it is crucial to selectively remove inessential information for better learning in the future ones. Empirically, we implement memory fading with a potential priority queue where the most previously learned tasks are the first to be removed. Instead of treating the seen tasks equally, we perform experiments on the impacts of graceful forgetting by managing the trade-off specific to part tasks.

Table \ref{Table_disposable} records the results of short-term memory (STM) on some early learned task(s) from FashionMNIST-10/5, compared to the original long-term memory (LTM) as a control group. It can be observed that the removal of declarative parameters on task 1, task 2, and tasks 1 \& 2 is beneficial to the remaining ones, somewhat akin to transfer learning. Concretely, we use $R_{T,t}$ to represent the test classification accuracy of a model on task $t$ after training on task $T$. The Acc of $R_{4,1}\!\!\sim\!\!R_{4,4}$ and $R_{5,1}\!\!\sim \!\!R_{5,5}$ are respectively 91.84\% and 91.05\% in the case of LTM. After forgetting task 1, the counterparts of both are improved. Besides, the more tasks a model forgets, the better performance on the rest. This can also be reflected in the experiments on CIFAR-100/5, reported in Table \ref{Table_disposable2}. Therefore, CLSNet with limited network capacity has the opportunity to train on more incoming tasks by erasing task-specific declarative parameters.

\section{Conclusion} \label{Conclusion}
Existing rehearsal-based approaches are mainly driven by an emphasis on rote memorization of old samples, rather than an understanding of how memorized data can influence the ways that connectionist models distinguish and remember previous knowledge. By contrast, this paper proposes a parameter-efficient class-incremental learning method called CLSNet, rethinking CIL from a parameter optimization perspective. Hence, it is simple yet effective to properly optimize the parameters of a model itself compared to buffering and retraining past observations cumulatively. Under this paradigm, a plastic CNN feature extractor and an analytical feed-forward classifier can be jointly optimized. Therefore, the proposed method holistically controls the parameters of a well-trained model such that the decision boundary learned fits new classes without losing its recognition of old classes. Extensive experiments show that the proposed CLSNet outperforms the selected comparison methods on five evaluation metrics and three benchmark datasets, in terms of accuracy gain, memory cost, training efficiency, and task-order robustness. Since there is no uniform standard for graceful forgetting in CIL and the benchmark remains an open question, this work preliminarily investigates its potential effects. Our future work will focus on the implementation of long/short-term memory retention for incrementally learning visual (image and video) tasks.



\ifCLASSOPTIONcaptionsoff
  \newpage
\fi



%

%
%
\bibliographystyle{IEEEtran}
\bibliography{IEEEabrv,mybibfile}

\end{document}